\documentclass{article}

\pdfoutput=1

\usepackage{arxiv}

\usepackage[utf8]{inputenc} 
\usepackage[T1]{fontenc}    
\usepackage{hyperref}       
\hypersetup{
     colorlinks=true,
     linkcolor=blue,
     filecolor=blue,
     citecolor = blue,      
     urlcolor=cyan,
     }
\usepackage{url}            
\usepackage{booktabs}       
\usepackage{amsfonts}       
\usepackage{nicefrac}       
\usepackage{microtype}      
\usepackage{lipsum}         
\usepackage{graphicx}
\usepackage{style_arXiv}
\usepackage{cleveref}       
\usepackage[numbers]{natbib}
\usepackage{doi}

\title{Distributed Non-Convex Optimization with One-Bit Compressors on Heterogeneous Data:\\ 
Efficient and Resilient Algorithms}

\author{ 
        Ming Xiang\\
	Department of Electrical and Computer Engineering\\
	Northeastern University\\
	Boston, MA 02115 \\
	\texttt{xiang.mi@northeastern.edu} \\
	\And
        Lili Su\\
	Department of Electrical and Computer Engineering\\
	Northeastern University\\
	Boston, MA 02115  \\
	\texttt{l.su@northeastern.edu} \\
}

\renewcommand{\shorttitle}{Distributed Non-Convex Optimization with One-Bit Compressors}

\hypersetup{
pdftitle={A template for the arxiv style},
pdfsubject={q-bio.NC, q-bio.QM},
pdfauthor={David S.~Hippocampus, Elias D.~Striatum},
pdfkeywords={First keyword, Second keyword, More},
}

\begin{document}
\maketitle

\begin{abstract}

Federated Learning (FL) is a nascent decentralized 
learning framework under which a massive collection of heterogeneous clients collaboratively train a model without revealing their local data. 
Scarce communication, privacy leakage, and Byzantine attacks are
the key bottlenecks of system scalability. 
In this paper, we focus on communication-efficient distributed (stochastic) gradient descent for non-convex optimization, a driving force of FL. We propose two algorithms, named {\em Adaptive Stochastic Sign SGD (Ada-StoSign)} and {\em $\beta$-Stochastic Sign SGD ($\beta$-StoSign)}, each of which compresses the local gradients into bit vectors. 

To handle unbounded gradients, Ada-StoSign uses a novel norm tracking function that adaptively adjusts a coarse estimation on the $\ell_{\infty}$ of the local gradients - a key parameter used in gradient compression. 
We show that Ada-StoSign converges in expectation with a rate $O(\log T/\sqrt{T} + 1/\sqrt{M})$, where $M$ is the number of clients. To the best of our knowledge, when $M$ is sufficiently large, Ada-StoSign outperforms the state-of-the-art sign-based method whose convergence rate is $O(T^{-1/4})$. 
Under bounded gradient assumption, $\beta$-StoSign achieves quantifiable Byzantine resilience and privacy assurances, and works with partial client participation and mini-batch gradients which could be unbounded. 
We corroborate and complement our theories by experiments on MNIST and CIFAR-10 datasets. 
\end{abstract}

\section{Introduction}
Federated Learning (FL) is a nascent learning framework that enables heterogeneous clients, under the coordination of a parameter server (PS), to collectively train a model without disclosing their raw data \cite{mcmahan2016federated,kairouz2021advances}. 
Expensive communication overhead and non-IID local data are two defining characteristics of FL. 
Moreover, a FL system is often implemented in harsh environment -- leaving the clients vulnerable to privacy leakage and Byzantine faults \cite{1996distributedlynch,chen2017distributed,blanchard2017machine,xie2018generalized,xie2019zeno,xie2020fall}.

A variety of communication-saving techniques have been introduced. 
However, challenges remain. 
Specifically, FedAvg -- the most widely-adopted FL algorithm -- 
saves communication via performing multiple local updates at the client side \cite{mcmahan2016federated,wang2019adaptive,stich2018local,li2020federated}. 
Large mini-batch size is another communication-saving technique yet its performance 
turns out be often inferior to FedAvg \cite{lin2018don}. 
Due to full gradients/updates transmission, both FedAvg and large mini-batch are still communication expensive. 
Gradient compressors \cite{xu2020compressed} take the physical layer of communication into account and are used to reduce the number of bits used in encoding local gradient information.  
Quantized SGD (QSGD) \cite{alistarh2017qsgd} is a lossy compressor with provable trade-off between the
number of bits communicated per iteration with the variance added to the process. 
The performance of QSGD is shown to be inferior to simple sign-based compressor such as SignSGD \cite{bernstein2018signsgd,bernstein2018signsgdmajority} for IID data, which
compresses a local gradient into a bit vector based on the sign of each coordinate.  
For IID data, SignSGD \cite{bernstein2018signsgdmajority} also enjoys nice resilience property -- it can tolerate up to 1/2 clients to be Byzantine. Nevertheless, signSGD fails to converge in the presence of non-IID data \cite{safaryan2021stochastic,chen2020distributed} due to its neglection of the gradient magnitude.  
\cite{jin2020stochastic} proposed a simple yet elegant method named stochastic sign SGD and analyzed its correctness on non-IID data. 
Unfortunately, no explicit convergence rate was given and their analysis (e.g.\,\cite[Theorem 6]{jin2020stochastic} ) contains major flaws, and there is no easy fix. 
Moreover, their method requires both the true gradient and the stochastic gradients to be bounded. 
It is worth noting when the local data is non-IID, even if the gradients are bounded, it is impossible for a system to tolerate up to 1/2 clients to be Byzantine \cite{mendes2015multidimensional}.  
\cite{safaryan2021stochastic} proposed a momentum version of sign-based method which converges with a rate $O\pth{T^{-1/4}}$. 
Both of our algorithms converge much faster when the client population is sufficiently large. 

\paragraph{Contributions.} 
In this paper, we consider non-IID data and and focus on communication-efficient distributed (stochastic) gradient descent for non-convex optimization.  
We propose two algorithms, named {\em Adaptive Stochastic Sign SGD (Ada-StoSign)} and {\em $\beta$-Stochastic Sign SGD ($\beta$-StoSign)}, each of which compresses the local gradients into bit vectors. 
To the best of our knowledge, we are the first to provide explicit convergence rates for sign-based methods in the presence of non-IID data. 
\begin{itemize}
\item To handle unbounded gradients, Ada-StoSign uses a novel norm tracking function that adaptively adjusts the a coarse estimation 
on the $\ell_\infty$ norm of the local gradients -- a key parameter used in gradient compression. We show that Ada-StoSign converges in expectation with a rate $O(\log T/\sqrt{T} + 1/\sqrt{M})$, where $M$ is the number of clients. When $T$ is fixed and $M=\Theta(T)$, the rate becomes $O(\log T/\sqrt{T})$ which matches that of the SGD for optimizing nonconvex functions \cite{ghadimi2013stochastic} up to polylog factor. 

\item Under the popular bounded gradient assumption, $\beta$-StoSign achieves quantifiable Byzantine resilience and privacy assurances, and works with partial client participation and mini-batch gradients which could be unbounded (Theorem \ref{thm: convergence subgaussian}). Both static and adaptive adversaries are considered. 
We show (in Theorem \ref{thm: sto sign dp necessity of beta}) that when $\beta=0$, the compressor is not differentially private. In sharp contrast, when $\beta>0$, the compressor is $d\cdot \log\pth{(2B+\beta)/\beta}$-differentially private, where $B$ is the bound of the true gradients.  
Notably, our gradient compressor with $\beta=0$ coincides with the compressor proposed in \cite{jin2020stochastic}. Although we admit the two compressors are structural alike, our contributions are significant and multifold. We reserve a point-by-point comparison in Appendix \ref{sec: comparison with jin}.
\item  Our theoretical findings are validated with experiments on the MNIST and CIFAR-10 datasets. In addition, our experiments show that Ada-StoSign works well with mini-batch gradients. 
\end{itemize}

\section{Problem Setup}
\label{sec: problelm setup}
%
%
The system consists of one parameter server (PS) and $M$ clients that collaboratively minimize 
\begin{align}
\label{eq: FL obj}
\min_{w\in \mathbb{R}^d} \, \, F(w) &: =\frac{1}{M} \sum_{m=1}^M F_m(w),
\end{align}
where 
$F_m(w):= \mathbb{E}_{\mathcal{D}_m}[F_m(w, x, y)]$ is the local cost function at client 
$m\in [M] :=\{1, \cdots, M\}$ with the expectation taken over heterogeneous local data $(x,y)\sim \mathcal{D}_m$.  

\vskip 0.3\baselineskip

%

\noindent1) {\em Client unavailability.}
Clients are also heterogeneous
in their computation  speeds and communication channel conditions, which result in intermittent clients unavailability. To capture this, following the literature \cite{kairouz2021advances,li2019convergence,philippenko2020bidirectional}, instead of full participation, we assume that, in each iteration, a client successfully uploads its local update with probability $p$ independently across rounds, and independently from the PS and other clients. 

\vskip 0.3\baselineskip
\noindent2) {\em Mobile Byzantine attacks.}
In each iteration $t$, up to $\tau$ clients suffer Byzantine faults. Denote by $\mathcal{B}(t)\subseteq [M]$ the set of clients 
are Byzantine in iteration $t$, which is unknown to the PS. Let $\tau(t) = |\calB(t)|$. 
We refer to the clients in $\calB(t)$ as Byzantine clients at iteration $t$. 
We consider both static and adaptive system adversaries. 
In the former, 
the system adversary does not know client unavailability in each iteration; in the latter, 
the system adversary adaptively chooses $\mathcal{B}(t)$ accordingly to the client unavailability in each iteration $t$. 

\vskip 0.3\baselineskip
\noindent3) {\em Differential privacy.} 
We also aim to provide quantitative privacy protection. Towards this, we use differential privacy framework as in Definition \ref{def: d.p. definition}. 
\begin{definition}[Definition 2.4 \cite{dwork2014algorithmic}]
\label{def: d.p. definition}
For any $\epsilon>0$, a randomized algorithm $\calM$ with domain $\mathbb{N}^{|\calX|}$ is $\epsilon$-differentially private if $\prob{\calM(x)\in\calS}\le\exp(\epsilon)\prob{\calM(y)\in\calS}$ holds  for all $\calS\subseteq$Range$(\calM)$ and for all $x,y\in\mathbb{N}^{|\calX|}$ such that $\|x-y\|_1\le1$. 
\end{definition}
\section{Algorithms}
\label{sec: alg}
We propose two sign-based methods: {\em Adaptive Stochastic Sign SGD (Ada-StoSign)}, described in Algorithm \ref{alg: alg 1} and {\em $\beta$-Stochastic Sign SGD ($\beta$-StoSign)}, described in Algorithm \ref{alg: alg 2}.  

\paragraph{Ada-StoSign.} 
Similar to {\em Stochastic Sign Descent with Momentum (SSDM) \cite{safaryan2021stochastic}}, we first consider full client participation and Byzantine-free setup, i.e., $p=1$ and $\tau=0$. 

\vskip 0.3\baselineskip
We use a special levelling rule (described in the box) to ensure $\frac{1}{2} B_1 \le \max_{m\in [M]} \|\nabla F_m(0)\|_{\infty}\le B_1.$ It is easy to see that the run-time is $\Theta(|\log_2 B_0/\max_{m\in [M]}\|\nabla F_m(w_0)\|_{\infty}|)$. 

\fbox{\begin{minipage}{35em}
{\em 
The PS sends $w_0$ and $B_0$ to each of the $M$ clients; \\
Each client $m\in [M]$ computes $\nabla F_m(w_0)$, and set $B_{\frac{1}{2}} \gets B_0$;\; \\
\For{$k=1, 2, \cdots$}
{
\For{each client $m\in [M]$}
{
If $\|\nabla F_m(w_0))\|_{\infty}>B_{\frac{1}{2}}$, sends ``level-up'' request to the PS; \\
If $\|\nabla F_m(w_0))\|_{\infty}<\frac{1}{2}B_{\frac{1}{2}}$, sends ``level-down'' request to the PS; 
 }
 At the PS:
 If at least a ``level-up'' request is received, then sends `$+$' to each client; \\
 If receives $M$ ``level-down'' requests, then sends `$-$' to each client; \\
 Otherwise,  set $B_1 \gets B_{\frac{1}{2}}$ and {\bf break}; \\
\For{each client $m\in [M]$}
{
If receives `$+$', set $B_{\frac{1}{2}} \gets 2B_{\frac{1}{2}}$\; \\
If receives `$-$', set $B_{\frac{1}{2}}\gets \frac{1}{2}B_{\frac{1}{2}}$\; \\
Otherwise, set $B_1 = B_{\frac{1}{2}}$ and {\bf break}. 
}
}
}
\end{minipage}}

\noindent
As described in line 4 of Algorithm \ref{alg: alg 1}, for  general $t\ge 1$, the PS and the $M$ clients collaboratively adjust the value $B_{t+1}$. With carefully chosen stepsize and the smoothness of global objective (Assumption \ref{ass: 2 smmothness}), we show in Theorem \ref{thm: level bound l infinty} that either $\max_{m\in [M]}\|\nabla F_m(w(t))\|_{\infty}$ is small (decaying in $t$) or
it can be closely tracked by $B_{t+1}$ up to $1/2$ multiplicative factor. In either case, $\max_{m\in [M]}\|\nabla F_m(w(t))\|_{\infty}\le B_{t+1}$. Hence, the sign compression probability in 7 is valid (i.e., with value in $[0,1]$). 
With this compression, the vector $\hat{g}_m\in \{\pm 1\}^d$. In line 11, the PS aggregates the received sign gradient vectors coordinate-wise via majority vote, breaking ties arbitrarily. 

\begin{algorithm}[htbp]
\setstretch{0.8}
\caption{Ada-StoSign}\label{alg: alg 1}
    \KwIn{$T, \eta_t, B_0$, and $w_0$} 
    \KwOut{$w(T)$} 
    \DontPrintSemicolon
    \SetNoFillComment
    \vspace{0.5em}
    
    {\bf Initialization:} The PS sends $w_0$ and $B_0$ to all clients. 
    
    \For{$t=0, \cdots, T-1$}
    {
        Each client $m$ computes $\nabla F_m(w(t))$;\; 
    
    \tcc{\color{olive} The PS and the $M$ clients collaboratively runs }
    \uIf{$t=0$}{run leveling rule in the box.}
    \Else{\SetKwFunction{FMain}{Norm Tracking} $B_{t+1}\gets$ \FMain{$t, B_t,\{\nabla F_m(w(t))\}_{m\in [M]}$}}
    
    \tcc{\color{olive}On each client $m\in [M]$}
    \For{each client $m\in [M]$}
    {\For{$i=1, \cdots, d$}{
    $\hat{g}_{mi}(t)\gets 1$ with probability $\frac{B_{t+1}+ \nabla F_m(w(t))}{2B_{t+1}}$; $\hat{g}_{mi}(t)\gets-1$ otherwise. 
    }
    Report $\hat{g}_{m}(t)$ to the PS;\;  
    }

    \tcc{\color{olive} On the PS}
    The PS, upon receipt of $\hat{g}_{m}(t)$s,  updates  
  $\tilde{\bm{g}}(t) ~\gets ~\sign\pth{\agg_{\text{maj}}\{\hat{g}_{m}(t): m \in [M]\}}$
    
    Broadcast $\tilde{\bm{g}}(t)$ to all clients;\;
    
    \tcc{\color{olive}On each client $m\in [M]$}
    \For{each client $m\in [M]$}
    {$w(t+1)\gets w(t)-\eta\tilde{\bm{g}}(t)$;\; }
    }
    \rule{\linewidth}{0.5pt}
     \SetKwFunction{FMain}{Norm Tracking}
    \SetKwProg{Fn}{Function}{:}{}
    \Fn{\FMain{$t, B_t, \{\nabla F_m(w(t))\}_{m\in [M]}$}}{{\em Initialization:} $\bm{u}^-\gets \mathbf{0}_M,$ $\bm{u}^\uparrow\gets \Indc_M,$ $\bm{u}^\downarrow\gets \Indc_M,$ $s^-\gets 0,$ $s^\uparrow\gets 0,$ $s^\downarrow\gets 0.$

\vskip \baselineskip
    \begin{multicols}{2}
    \tcc{\color{olive} On each $m\in[M]$}
\For{each client $m\in [M]$}
{
 $\bm{u}^-_m=  \bm{1}\sth{\linf{\nabla F_m(w(t))}< \frac{5c}{\sqrt{t+1}} } $\;

 $\bm{u}^\uparrow_m  = \bm{1}\sth{\linf{\nabla F_m(w(t))}  >B_t}$,
 $\bm{u}^\downarrow_m =  \bm{1}\sth{\linf{\nabla F_m(w(t))} < \frac{1}{2}B_{t}}$

    Report $\bm{u}^-_m,\bm{u}^\uparrow_m,\bm{u}^\downarrow_m$ to the PS;\;  
}
     \tcc{\color{olive} On the PS}

    
    The PS, upon receiving $\bm{u}^-_m,\bm{u}^\uparrow_m,\bm{u}^\downarrow_m$, do    
    
    \uIf{$\bm{u}^-=\Indc_M$}
    {
    {$s^-\gets 1$\;}

    }
    \uElseIf{$\linf{\bm{u}^\uparrow}=1$}{
    $s^\uparrow\gets 1$\;
    }
    \uElseIf{$\bm{u}^\downarrow=\Indc_M$}{
    $s^\downarrow\gets 1$\;
    }
    \tcc{\color{olive} On each $m\in[M]$}
    \For{each client $m\in [M]$}
    {
    Upon receipt of $s^-,s^\uparrow,s^\downarrow$, do 
    
    \uIf{$s^-=1$}{
    $B_{t+1}=\frac{5c}{\sqrt{t+1}}$\;
    
    {\bf Break}
    }
    \uElseIf{$s^\uparrow=1$}{
    $B_{t+1}=2 B_{t}$\;
    }
    \uElseIf{$s^\downarrow=1$}{
    $B_{t+1}=\frac{1}{2} B_{t}$\;
    }
    }
    \end{multicols}
    
    }
    \text{\bf End Function}
\end{algorithm}

\vskip 0.5\baselineskip
\noindent{\em \underline{Norm Tracking Function.}}
We provide its high level idea as follows. 
When the function is called, first via exchanging bits, the clients and PS collectively determine if 
$\max\limits_{m\in[M]}\linf{\nabla F_m(t)}<\frac{5c}{\sqrt{t+1}}$
holds. If so, then set $B_{t+1}=\frac{5c}{\sqrt{t+1}}$. 
Otherwise, we perform one-round in the {\bf for-loop} of the levelling rule at the first iteration (the code in the displayed box). 
Specifically, if the PS receives a ``level up'' request, then the PS will inform the clients to increase $B$; consequently,  
$B_{t+1} = 2B_t$. If the PS receives $M$ ``level down'' requests, then the PS will inform the clients to decrease $B$; 
consequently, $B_{t+1} = \frac{1}{2}B_t$. Otherwise, $B(t+1) = B(t)$. 
The function is formally described in lines 21 - 45 of Algorithm \ref{alg: alg 1}. 

The vector $\bm{u}^-$ is used to determine whether $\max\limits_{m\in[M]}\linf{\nabla F_m(t)}<\frac{5c}{\sqrt{t+1}}$ is true. 
The vectors $\bm{u}^\uparrow$ and $\bm{u}^\downarrow$ are used to collect the ``level up'' and ``level down'' requests from the clients. 
The three variables $s^-, s^{\uparrow}, s^{\downarrow}$ are used to encode the leveling decision aggregated by the PS.

\begin{remark}[mini-batch]
In Algorithm \ref{alg: alg 1}, each client uses true local gradients. 
Our experimental results Fig.\,\ref{fig: adaptive B} in Section \ref{sec: numerical experiments} shows that, despite the randomness in the mini-batch causes the $B$ to update more frequently than using true gradients, both the training errors and test accuracy of mini-bath are comparable to true gradients with minimal performance degradation. We would like to explore the theoretical analysis of the mini-batch convergence in a follow-up work. 
\end{remark}

\paragraph{$\beta$-StoSign.} 

%
%
The vote rule in the norm tracking function is vulnerable to Byzantine adversary.
This motivates $\beta-$StoSign (formally described in Algorithm \ref{alg: alg 2}), where we uses a clipping function to ensure the validity of the probability in generating the bit vector. 
\begin{definition}
\label{def: clip function}
The clipping function with parameter $B$, denoted by $\clip\sth{\cdot,B}$, projects $g\in \reals$ onto $[-B, B]$ as $\clip\sth{g,B}=\max\sth{-B,\min\sth{B,g}}.$
\end{definition}

Compared with Ada-StoSign, Algorithm \ref{alg: alg 2} takes in two additional parameters: {$\beta\ge0$} and $n$; the former is the privacy budget and the latter is the mini-batch size. 
Depart from  Ada-StoSign, $\beta$-StoSign can handle both Byzantine attacks (i.e., $\tau>0$) and partial clients (i.e., $p<1$).  
In each iteration $t$, a client $m$ is selected by the PS with probability $p$. Let $\calS(t)$ be the set of selected clients at time $t$. Since Byzantine clients can deviate from Algorithm \ref{alg: alg 2} arbitrarily, lines 2-8 are executed at 
 clients in $\calS(t)\setminus \calB(t)$ only. In each iteration $t$, client $m\in \calS(t)\setminus \calB(t)$ first obtains $n$ stochastic gradients $\bm{g}_{m}^1(t), \ldots, \bm{g}_{m}^n(t)$. 
 Then it passes $\frac{1}{n}\sum_{j=1}^n\bm{g}_{m}^j(t)$ to 
 $\clip\sth{\cdot,B}$ coordinate-wise, 
and compresses the clipped gradient into $\{\pm 1\}$.  
For convenience of exposition, if no message is received from a selected client $m$ (which only occurs when $m\in \calB(t)$), then the PS treats $\hat{u}_m$ as $\bm{0}$. 
Notably, if $m\in \calB(t)$, the received $\hat{u}_m(t)$
could take arbitrary value. 
Since $\hat{g}_m\in \{\pm 1\}^d$ for $m\notin \calB(t)$,  if $\hat{u}_{mi}(t)\notin \{-1,1\}$, then it must be true that client $m\in \calB(t)$. Thus, $\hat{u}_{mi}(t)$ will be removed from aggregation by the PS. In other words, it is always a better strategy for a Byzantine client to restrict $\hat{u}\in \{\pm 1\}^d$. Henceforth, without loss of generality, we assume that $\hat{u}_m(t)\in \{\pm 1\}^d$ for all received compressed gradients.

\begin{algorithm}[hbtp]
\setstretch{0.2}
    \caption{Distributed Non-Convex Optimization with $\beta$-Stochastic Sign SGD}\label{alg: alg 2}
    \KwIn{$T, \eta, \beta, n, B$, and $\nu$} 
    \KwOut{$w(T)$} 
    \DontPrintSemicolon
    \vspace{0.5em}
    
    {\bf Initialization:} $w(0)\gets \nu$ for each $m\in [M]$, and 
    the PS samples each client $m\in [M]$ with probability $p$ to form $\calS(0)$; 
    
    \For{$t=0, \cdots, T-1$}
    {
    \tcc{\color{olive} On each $m\in \calS(t)\setminus \calB(t)$}
    
    Get $n$ stochastic gradients $\bm{g}_{m}^1(t), \ldots, \bm{g}_{m}^n(t)$;\; 
    
    \For{$i=1, \cdots, d$}{
    $\hat{g}_{mi}(t)\gets 1$ with probability $\frac{B+\beta+ \clip\sth{\frac{1}{n}\sum_{j=1}^n\bm{g}_{mi}^j(t),B}}{2B+2\beta}$; $\hat{g}_{mi}(t)\gets-1$ otherwise. 
    }
    
    Report $\hat{g}_{m}(t)$ to the PS;\;  
    
    \tcc{\color{olive} On the PS}
    
    Wait to receive messages $\hat{u}_{m}(t)\in \reals^d$ from the sampled clients $\calS(t)$\; 
    
  $\tilde{\bm{g}}(t) ~\gets ~\sign\pth{\agg_{\text{maj}}\{\hat{u}_{m}(t): m \in \calS(t)\}}$
    
    Sample each client $m\in [M]$ with probability $p$ to obtain $\calS(t+1)$;\; 
    Broadcast $\tilde{\bm{g}}(t)$ to all clients;\;

    \tcc{\color{olive}On each client $m\in \calS(t+1)\setminus \calB(t)$}
    
    Upon receiving $\tilde{\bm{g}}(t)$: $w(t+1)\gets w(t)-\eta\tilde{\bm{g}}(t)$;\;    
    }
\end{algorithm}
\vskip -\baselineskip

\section{Main Results}
Our analysis is derived under the following technical assumptions that are standard in non-convex optimization  \cite{shalev2014understanding}. 

\begin{assumption}[Lower bound]
\label{ass: 1 lower bound}
There exists $F^*$ such that 
$F(w)\geq F^*$ for all $w$. 
\end{assumption}

\begin{assumption}[Smoothness]
\label{ass: 2 smmothness}
There exists some non-negative constant $L$ such that \newline $F(w_1)\leq F(w_2)+\langle\nabla F(w_2),w_1-w_2\rangle+\frac{L}{2}{ \|w_1-w_2\|_2^2}$ for all $w_1, w_2$.
\end{assumption}

\begin{assumption}[Bounded dissimilarity]
\label{ass: BG condition}
There exists $\tilde{B}>0$ and $\tilde{G}>0$ such that 
$\|\nabla F_m(w) - \nabla F(w)\|_2 \le \tilde{B} \|\nabla F(w)\|_2 + \tilde{G} ~ \forall ~ w$ for each $m\in [M]$. 
\end{assumption}

\begin{assumption}[Bounded true gradient]
\label{ass: 4 Bounded gradient}
For any coordinate $i\in [d]$, there exists $B_i>0$ such that $\abth{\nabla f_{mi}(w)}\le B_i$ for all $m\in[M]$. 
Let $B_0: = \max_{i\in [d]}B_i$.
\end{assumption}
\begin{assumption}[Sub-Gaussianity]
\label{ass: sub gaussianity1}
For a given client $m\in [M]$, at any query $w\in\mathbb{R}^d$, the stochastic gradient $\bm{g}_m(w)$ 
is an independent unbiased estimate of $\nabla F_{m}(w)$ that is coordinate-wise related to the gradient $\nabla F_m(w)$ as $\bm{g}_{mi}(w)=\nabla F_{mi}(w)+\bm{\xi}_{mi} ~ \forall\, i\in [d],$ 
where $\bm{\xi}_{mi}$ is zero-mean $\sigma_{mi}$-sub-Gaussian, i.e,
$\expect{\bm{\xi}_{mi}}=0, $ and the two deviation inequalities $\prob{\bm{\xi}_{mi}\ge t}\le\exp\pth{-\frac{t^2}{2\sigma_{mi}^2}}$ and $\prob{\bm{\xi}_{mi}\le { -t} }\ge\exp\pth{-\frac{t^2}{2\sigma_{mi}^2}}$
 hold. Let $\sigma^2 : = \max_{m\in [M], i\in [d]} \sigma^2_{mi}$. 
\end{assumption}
{
\begin{assumption}[Heavy-tailed noise]
\label{ass: heavy tailed}
Let $\bm{\xi}_{mi}$ defined in Assumption \ref{ass: sub gaussianity1} be a zero-mean random variable, $\expect{\bm{\xi}_{mi}^{2}}\le \sigma^2$, and $\expect{\abth{\bm{\xi}_{mi}}^{p^{\prime}}}\le M_{p^\prime}<\infty$ for $p^\prime\ge 4$.
\end{assumption}
}

Before we start our analysis, we first define the random time $R$.
\begin{definition}
    Let $R$ be the random time with a probability mass function
\begin{align*}
    \prob{R=k} = \frac{\eta_k}{\sum_{t=0}^{T-1}\eta_t},\quad k=0,\ldots,T-1.
\end{align*}
\end{definition}
Following the road-map used in  \cite{bernstein2018signsgdmajority,jin2020stochastic,safaryan2021stochastic}, we first establish an upper bound for the probability of gradient sign errors $\prob{\tilde{\bm{g}}_i(t)}\not=\sign\pth{{ \nabla F_i}(w(t))}$, and then bound $\sum_{t=0}^{T-1}\eta_t\expect{\|\nabla F(w(t))\|_1}$ to conclude convergence. 
Recall that $\bm{g}_m (t):=\nabla F_m(w(t))$ denotes the true local gradient at client $m$. We will use the two notations interchangeably to denote the true local gradient afterwards.

\subsection{Ada-StoSign}
\begin{proposition}
\label{prop: iteration perturbation: l infinity}
Choose $\eta_t = \frac{c}{L\sqrt{d(t+1)}}$. For $t\ge 0$, if $\max_{m\in [M]} \|\nabla F_m(w(t))\|_{\infty} >  \frac{2c}{\sqrt{t+1}}$, then 
\[
\frac{1}{2} \max_{m\in [M]} \|\nabla F_m(w(t))\|_{\infty} < \max_{m\in [M]}\|\nabla F_m(w(t+1))\|_{\infty} < \frac{3}{2} \max_{m\in [M]} \|\nabla F_m(w(t))\|_{\infty}. 
\]
\end{proposition}
\begin{theorem}
\label{thm: level bound l infinty}
Suppose that Assumption \ref{ass: 2 smmothness} holds. Then 
\begin{align*}
\frac{1}{2}B_{t+1} \indc{\max_{m\in [M]}\|\nabla F_m (w(t))\|_{\infty} \ge \frac{5 c}{\sqrt{t+1}}} 
\le \max_{m\in [M]}\|\nabla F_m(w(t))\|_{\infty} \le B_{t+1}. 
\end{align*}
\end{theorem}
\begin{theorem}
\label{thm: prob of sign disagreement}
Suppose Assumption \ref{ass: 2 smmothness} holds.  For any given $c_0\in (0, 1)$, choose $c_1:=\sqrt{\log\pth{\frac{2}{1-c_0}}}$. When $\abth{\nabla F_i (t)}\ge c_1B_{t+1}\sqrt{\frac{2}{M}}$ for $i\in[d]$, 
we have
    $$\prob{\frac{1}{M}\sum_{m=1}^M \sign\pth{\hat{g}_{mi}(t)}\neq  \sign\pth{\frac{1}{M}\sum_{m=1}^M\nabla F_{mi}(t)} \mid \calF_t}\le \frac{1-c_0}{2},$$
    where $\calF_t$ is the natural filteration. 
\end{theorem}
\begin{theorem}
\label{thm: convergence alg 1}
Suppose Assumption \ref{ass: 1 lower bound}, \ref{ass: 2 smmothness} and \ref{ass: BG condition} hold. Let $\Delta=F(w(0))-F^*$ and a decaying learning rate $\eta_t=\frac{1}{L\sqrt{d(t+1)}}.$ Recall that $R$ is the random time. For any  $c_0\in (\frac{1}{2}, 1)$, let $M$ be sufficiently large such that $c_0-4d\sqrt{\log\pth{\frac{2}{1-c_0}}\frac{2}{M}}\pth{\tilde{B}+1}\ge\frac{1}{2},$ then 
\begin{small}
    \begin{flalign}
\label{eq: alg 1 convergence}
   & \expect{\norm{\nabla F(R)}} \le O \pth{\frac{{2{\Delta}\sqrt{d}}}{L\sqrt{T}} + 8d \sqrt{\frac{2}{M}\log\pth{\frac{2}{1-c_0}}}\tilde{G} + \left(1 + 12\sqrt{\frac{2d}{M}\log\pth{\frac{2}{1-c_0}}}\right)\frac{\sqrt{d}\log T}{\sqrt{T}}}.&
\end{flalign}
\end{small}

\end{theorem}
\begin{remark}
The condition given in Theorem \ref{thm: convergence alg 1} is rather mild. Recall that our distributed optimization is deployed over a huge system at scale. As $M\rightarrow \infty,$ the second term approaches 0. 

When $M\ge T$, the convergence rate is $O(\log T/\sqrt{T})$, which is significantly better than $O(\frac{1}{T^{1/4}})$ -- the rate of SSDM and matches that of the SGD for optimizing non-convex function polylog factors \cite{ghadimi2013stochastic}. 
Furthermore, in the iid case when $\tilde{G}=0$, we recover the standard SGD convergence rate $O\pth{\frac{1}{\sqrt{T}}}$ up to polylog factors. 
\end{remark}

\subsection{$\beta$-StoSign}
%
%
We characterize the DP of our gradient compressor $\beta$-StoSign. 
Over the entire training time horizon, the quantification of the differential privacy preserved for 
any given client  can be obtained by applying the composition theorem of $\epsilon$-differentially private algorithms \cite[Corollary 3.15]{dwork2014algorithmic}. { We first show  that $\beta$ is an enablor of DP for our compressor. }

\begin{theorem}
\label{thm: sto sign dp necessity of beta}
$0$-StoSign is not differentially private. 
That is, there does not exist a finite $\epsilon>0$ for which Definition \ref{def: d.p. definition} holds. 
When $\beta >0$, $\beta$-StoSign 
is $d\cdot\log\pth{\frac{2B+\beta}{\beta}}$-{ DP} for all gradients. 
\end{theorem}
Theorem \ref{thm: sto sign dp necessity of beta} also implies that as long as { $\beta >0$}, $\beta$-StoSign ensures $\epsilon$-differential privacy for $\epsilon = O(d)$ in any iteration $t$. 
We defer more refined DP characterizations to Appendix \ref{app: privacy}.

Henceforward, we present the convergence results under a unified framework, where $\Xi(n)=\Xi_1(n)=2(B+\beta)\exp\pth{-\frac{n}{2}}$ in the case of sub-Gaussian noise, and $\Xi(n)=\Xi_2(n)=\frac{4(B+\beta)}{n^{\frac{p^\prime}{2}}}$ in the case of heavy-tailed noise. 
For the analysis on static and adaptive Byzantine adversaries, we note that we have two cases. Specifically, when the system adversary is static but with $\tau(t)\le \frac{2}{p^2}\log\frac{6}{c}$ or when is adaptive, let $\Upsilon(M,t)=\Upsilon_1(M,t)=\frac{2(B+\beta)\tau(t)}{pM}.$ On the other hand, when the system adversary is static with $\tau(t) >\frac{2}{p^2}\log\frac{6}{c}$, let $\Upsilon(M,t)=\Upsilon_2(M,t)=\frac{3(B+\beta)\tau(t)}{M}.$ We further define $\delta(M)=\frac{2(B+\beta)}{p}\sqrt{\frac{2}{M}\log\frac{6}{3-5c}}+2\sigma\sqrt{\frac{2\log\frac{6}{c}}{{Mn}}}$ and let $c\in(0,\frac{3}{5}).$
{
\begin{theorem}
\label{thm: sub-gaussian} 
Fix $t\ge 1$ and $i\in [d]$ and choose $B=\pth{1+\epsilon_0}B_0$. It holds that
%
\begin{align}
\label{eq: error bound}
\prob{\tilde{\bm{g}}_i(t)\neq\sign\pth{{ \nabla F}(w(t))_i}\mid w(t)}\le \frac{1-c}{2}.
\end{align}

\noindent{\bf (Sub-Gaussian noise):} Suppose Assumption \ref{ass: 4 Bounded gradient} and \ref{ass: sub gaussianity1} hold, and $\epsilon_0>\frac{\sigma}{B_0}$. Eq.\,\eqref{eq: error bound} holds if $\abth{{ \nabla F}(w(t))_i} \ge \delta(M)+ \Xi_1(n) +\Upsilon(M,t).$

\noindent{\bf (Heavy-tailed noise):} Suppose Assumption \ref{ass: 4 Bounded gradient} and \ref{ass: heavy tailed} hold, and $\epsilon_0> \frac{\sqrt[p^\prime]{M_{p^\prime}}}{B_0}$. Eq.\,\eqref{eq: error bound} holds for $p^\prime\ge 4$ if $\abth{{ \nabla F}(w(t))_i} \ge \delta(M)+ \Xi_2(n) + \Upsilon(M,t).$

%
\end{theorem}
}

Theorem \ref{thm: sub-gaussian} says that when $\abth{{ \nabla F}(w(t))_i}$ is large enough, the sign estimation at the PS in each iteration is more likely to be correct. This is crucial in ensuring the convergence because Theorem \ref{thm: sub-gaussian} implies that when $\abth{{ \nabla F}(w(t))_i}$  is large enough, in expectation, Algorithm \ref{alg: alg 2} pushes $w(t)$ towards a stationary point of the global objective $F$. {In addition, the distributed algorithm reduces to centralized signSGD given a correct population sign.} After all, small $\abth{{ \nabla F}(w(t))_i}$  implies that $w(t)$ is already near the neighborhood of a stationary point. Different from \cite{safaryan2021stochastic} and \cite{jin2020stochastic}, we neither assume the sign error distributions across clients be identical,  nor require the average probability of sign error to be less than $1/2$. Instead, we show that it is enough to { let} the probability of population sign errors be small when the magnitude of the gradients is large. 

\begin{theorem}
\label{thm: convergence subgaussian}
\noindent Suppose Assumptions \ref{ass: 1 lower bound}, \ref{ass: 2 smmothness}, \ref{ass: 4 Bounded gradient} hold. For any given $t$, we choose $B=(1+\epsilon_0)B_0.$ With Assumption \ref{ass: sub gaussianity1} for $\epsilon_0>\frac{\sigma}{B_0}$ or \ref{ass: heavy tailed} for $\epsilon_0> \frac{\sqrt[p^\prime]{M_{p^\prime}}}{B_0}~(p^\prime\ge4)$, we have
\begin{align*}
    \expect{\|\nabla F(w(R))\|_1}\le \frac{1}{c}\left[
    \frac{F(w(0)) - F^*}{\sum_{t=0}^{T-1}\eta_t} + \frac{Ld\sum_{t=0}^{T-1}\eta_t^2}{2\sum_{t=0}^{T-1}\eta_t} + \delta(M)+2 d \Xi(n)+\frac{2d\sum_{t=0}^{T-1}\eta_t\Upsilon(M,t)}{\sum_{t=0}^{T-1}\eta_t}\right]&
\end{align*}

\end{theorem}

\begin{corollary}
\label{corollary: main text learning rate}
    Let $\Delta=F(w(0))-F^*.$ We have the following convergence rate for Algorithm \ref{alg: alg 2} under specific choices of the learning rates. Recall that $R$ is the random time. 
    \vskip 0.2\baselineskip
    \noindent{Suppose $\eta_t=\frac{1}{\sqrt{dT}}.$} Then, we have
    \begin{align}
    \label{eq: constant step size}
        \expect{\|\nabla F(w(R))\|_1}\le \frac{1}{c}\left[
        \frac{\Delta\sqrt{d}}{\sqrt{T}} + \frac{L\sqrt{d}}{2\sqrt{T}} + \delta(M)+2 d \Xi(n)+\frac{2d}{T}\sum_{t=0}^{T-1}\Upsilon(M,t)\right]&
    \end{align}
     {Suppose $\eta_t=\frac{1}{\sqrt{d(t+1)}}.$} Then, we have
    \begin{flalign}
    \label{eq: decay step size}
        &\expect{\|\nabla F(w(R))\|_1}\le O \pth{
        \frac{\Delta\sqrt{d}}{\sqrt{T}} + \frac{L\sqrt{d}\log T}{2\sqrt{T}} + \delta(M)+2 d \Xi(n)+2d\max_{t\in[T-1]}\Upsilon(M,t)}&
    \end{flalign}
\end{corollary}

\begin{remark}
\label{rmk: convergence results}
(1) The convergence rates in the two cases of Byzantine adversaries $\Upsilon_1$ and $\Upsilon_2$ differ only by a multiplicative factor of $\frac{3p}{2}.$ As long as $\tau(t)$ is sufficiently large, the impacts of $p,$ on the convergence rate upper bound is limited. The lower bound requirement on $\tau(t)$ might be an artifact of our analysis on static adversaries in simplifying the boundary case derivation. The residual term  for adaptive adversaries is $\Upsilon_1(M,t)$ only.

(2) Now consider the asymptotic in terms of  $T$ and the client number $M$ only. If $\tau(t)=\tau$ for each $t$, then the Byzantine terms in  Eq.\,\eqref{eq: constant step size} and Eq.\,\eqref{eq: decay step size} become $2d\Upsilon(M,0)$, where $2d\Upsilon_1(M,0)=\frac{4(B+\beta)\tau d}{pM}$ and $2d\Upsilon_2(M,0)=\frac{6(B+\beta)\tau d}{M}$.
If $\tau=O\pth{\sqrt{M}}$, then both terms scale in $M$ with order $O(\frac{1}{\sqrt{M}})$, which is of the same order as the term $\delta(M)$, which is the consequence of weak signal strength of the compressed gradients near a stationary point of the global objective $F$. 
On the other hand, if $\sum_{t=0}^{T-1}\tau(t)=O\pth{\sqrt{T}}$ in Eq.\,\eqref{eq: constant step size}, the Byzantine terms scale as $O(\frac{1}{\sqrt{T}})$, which is of the same order as the first two terms.  
In either case, due to the mobility of the Byzantine faults, it is possible that $\cup_{t=0}^T \calB(t) = [M]$, i.e., every client is { corrupted} at least once.

(3) The residual term $\Xi(n)$ is an immediate consequence of using mini-batch stochastic gradients instead of true gradient as in \cite{jin2020stochastic}. It turns out that this term have minimal impact on the final convergence. In fact, as long as $n=\Omega\pth{\log M}$ (sub-Gaussian noise) or $n=\Omega\pth{M^{\frac{1}{p^\prime}}}$ for $p^\prime\ge 4$ (heavy-tailed noise), this term becomes non-dominating.

(4) In Theorem \ref{thm: sto sign dp necessity of beta}, we know that $\beta$-StoSign is $\epsilon = d\cdot\log\pth{\frac{2B+\beta}{\beta}}$-DP. Simple algebra leads to $\beta=\frac{2B}{e^{{\epsilon}/{d}}-1}.$ As $\beta$ is in the numerator, $\epsilon$-DP protection worsens the bound. This is observed also in the experiments (Section \ref{sec: numerical experiments}.)

(5) When $\tau(t)=\tau = O\pth{\sqrt{M}}$ and $n$ of the same order as in (3), the convergence rates 
become $O(\frac{1}{\sqrt{T}} + \frac{1}{\sqrt{M}})$, approaching the convergence rate of the standard (centralized, non-private, and adversary-free) SGD $O\pth{\frac{1}{\sqrt{T}}}$ as $M\rightarrow \infty$.

\vskip -2\baselineskip
%
%
%
\end{remark}
The bounds in Theorem \ref{thm: convergence subgaussian} can be tightened with more structured gradient noises. We defer our results on Gaussian-tailed and bounded stochastic gradients to Appendix \ref{sec: additional convergence rates}. All of the results have a similar form and differ only in the noisy residual terms. 
%
%
\vskip -0.5\baselineskip
\section{Numerical Experiments}
\label{sec: numerical experiments}
In this section, we evaluate our analysis on Algorithm \ref{alg: alg 1} and \ref{alg: alg 2}.
The implementation details and testing accuracies are deferred to Appendix. \ref{app:experiments}. We list the key elements of our experimental setup for benchmark datasets below. All the experiment results are reported after $3$ repeated trials under different random seeds, unless otherwise noted.
\newline
\noindent{\bf Datasets:} MNIST \cite{lecun2009mnist}, and CIFAR-10 \cite{krizhevsky2009learning}.
    
\noindent{\bf Models:} Multinomial logistic classification, a CNN with two 5x5 convolution layers (the first with 64 channels, the second with 64 channels), and multi-layer perceptron (MLP) \cite{mcmahan2016federated}.
    
\noindent{\bf Clients Data:} 100 balanced workers with non-IID distributions to be specified.

\noindent{\bf Baseline Algorithms:} signSGD \cite{bernstein2018signsgd} and FedSGD \cite{mcmahan2016federated}.
\subsection{Algorithm \ref{alg: alg 1}}
\label{sec: experiments alg 1}
We use multinomial logistic classification with a decaying learning rate $\eta_t=\frac{\eta_0}{\sqrt{t+1}}$ and normalize the data so that the loss function is $1$-smooth. Each client holds images from two classes to make a highly non-identical data distribution. Although our theory presents the convergence for full-batch gradients only, the numerical results suggest Algorithm \ref{alg: alg 1} may work for mini-batch stochastic gradient, which we leave as a future direction.
\begin{figure}[!htb]
    \centering
    \includegraphics[width=\textwidth,trim=1cm 0.2cm 0.5cm 0cm,clip]{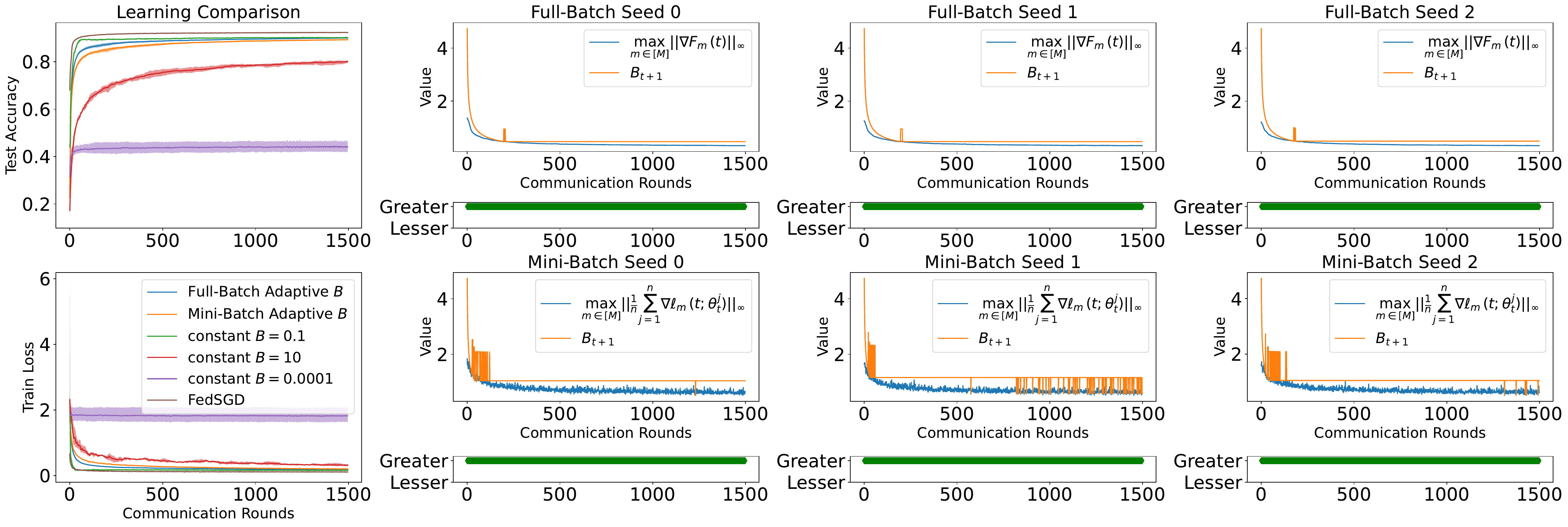}
    \caption{Learning Performance, $B_{t+1}$'s trajectory, and comparisons with gradient norms.}
    \label{fig: adaptive B}
\end{figure}
\vskip 0.5\baselineskip
Algorithm \ref{alg: alg 1} does not require a concrete value of $B_0$, offering users freedom to initialize their own $B_0$, i.e., $B$ in iteration $0.$
For a fair of evaluation and aligning with Theorem \ref{thm: level bound l infinty}, we randomly generate $B_1$ based on the gradients' $\ell_\infty$ norm. In reality, this involves back and forth collaborations between PS and clients and may sacrifice some communication efficiency; however, it is worth it. The benefits are two-folds: 1) Algorithm \ref{alg: alg 1} consistently achieves high test accuracy and smooth loss with no need to tune $B$. In contrast, Fig.\ \,\ref{fig: adaptive B} shows that a constant $B$ of different values may affect the learning performance. 2) the leveling rules ensure a vaild probability measure as observed from Fig.\ \,\ref{fig: adaptive B}.
This meets our analysis such that each client updates $B_{t+1}$ no more than once for $t\ge 1$. 

\subsection{Algorithm \ref{alg: alg 2}}
\subsubsection{Learning Performance}
\label{sec: alg 2 learning}
We evaluate Algorithm \ref{alg: alg 2} with a decaying learning rate $\eta_t=\frac{\eta_0}{\sqrt{t+1}}$ on MNIST via MLP and CIFAR-10 data sets via CNN. 
The most notable change in Algorithm \ref{alg: alg 2} from Algorithm \ref{alg: alg 1} is $B$ being a hyper parameter to be fine-tuned from extensive experiments, instead of an time-varying value. 
The additional details are deferred to Appendix \ref{app: experiments general}.

First, signSGD is consistently inferior to the other two algorithms. This might be due to its inability to deliver magnitude information.
Next, when $\beta=0$, i.e, when with no privacy protection, Algorithm \ref{alg: alg 2} converges faster than FedSGD and has a comparable train loss. Nevertheless, we do not anticipate a better performance of Algorithm \ref{alg: alg 2} than FedSGD since the latter transmits the uncompressed gradients to PS, while our gradient information is encoded in expectation.
As $\beta$ increases, we observe a performance drop in $\beta$-Stochastic SGD. This is intuitively and theoretically correct because $\beta$ will introduce additional noise and thus deviate the convergence bound. 
\begin{figure}[!htb]
    \centering
    \begin{subfigure}[b]{\textwidth}
    \includegraphics[width=\linewidth]{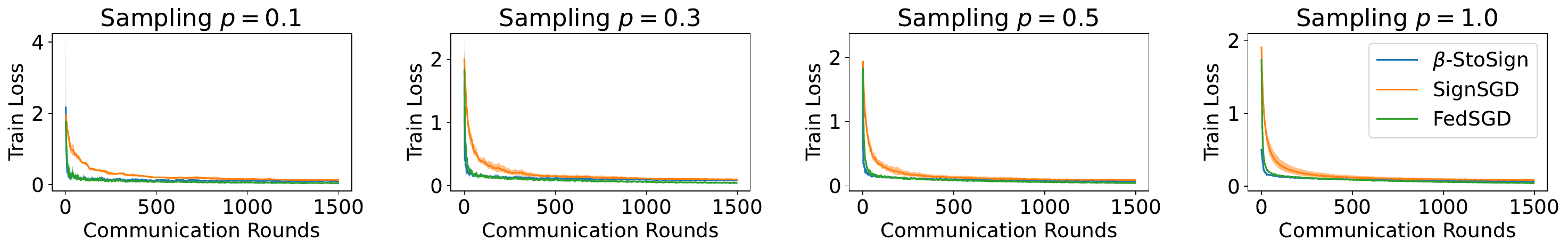}
    \caption{$\beta=0$}
    \end{subfigure}
    \begin{subfigure}[b]{\textwidth}
        \includegraphics[width=\linewidth]{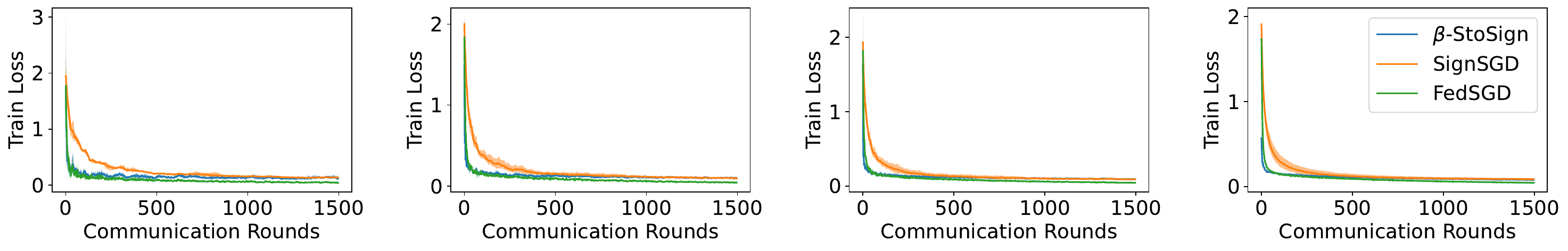}
        \caption{$\beta=B$}
    \end{subfigure}
    \begin{subfigure}[b]{\textwidth}
        \includegraphics[width=\linewidth]{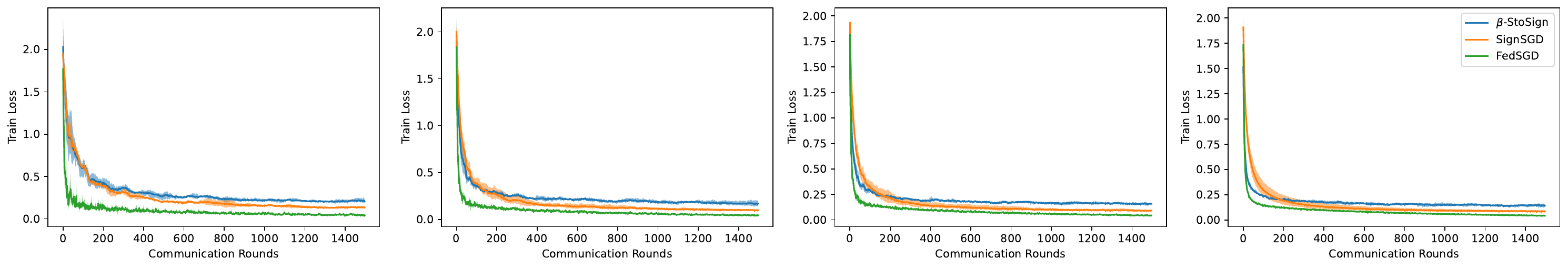}
        \caption{$\beta=10B$}
    \end{subfigure}
    \caption{Train loss comparisons on MNIST data set under non-IID data}
    \label{fig: MNIST learning mainbody}
\end{figure}

\begin{figure}
    \centering
    \begin{subfigure}[b]{\textwidth}
        \includegraphics[width=\textwidth]{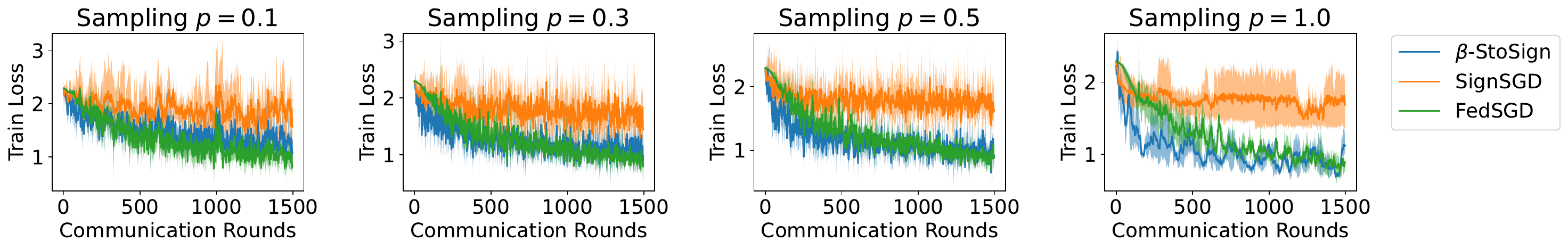}
        \caption{$\beta=0$}
    \end{subfigure}
    \begin{subfigure}[b]{\textwidth}
        \includegraphics[width=\textwidth]{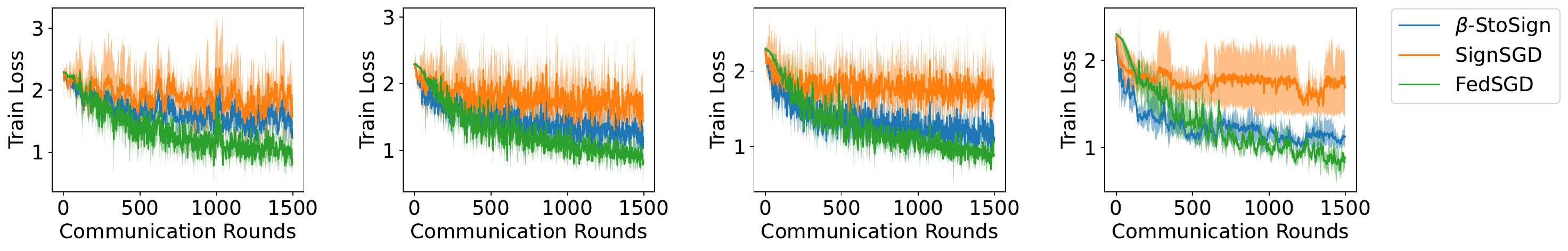}
        \caption{$\beta=B$}
    \end{subfigure}
    \begin{subfigure}[b]{\textwidth}
        \includegraphics[width=\textwidth]{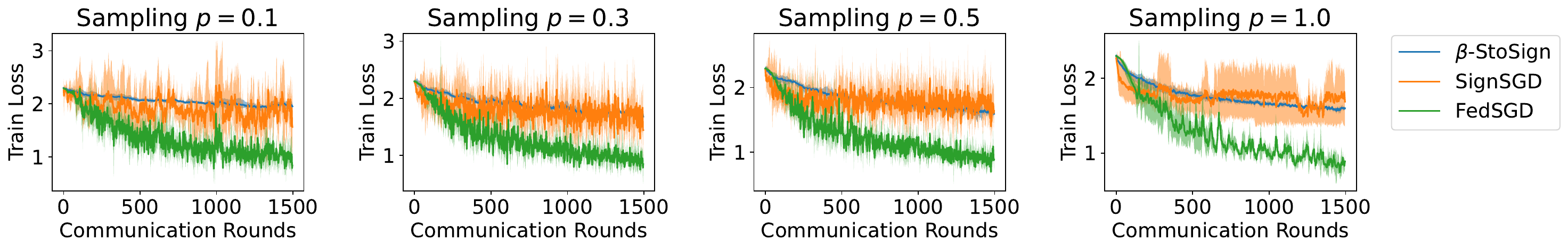}
        \caption{$\beta=10B$}
    \end{subfigure}
    \caption{Train loss comparisons on CIFAR-10 data set under non-IID data}
    \label{fig:CIFAR10 learning mainbody}
\end{figure}

\vskip -\baselineskip
\subsubsection{Byzantine adversaries}
\label{sec: alg 2 Byzantine}
\begin{figure}[!htb]
    \centering
        \includegraphics[trim={0cm 0cm 0.5 -1.5cm},clip,width=\textwidth]{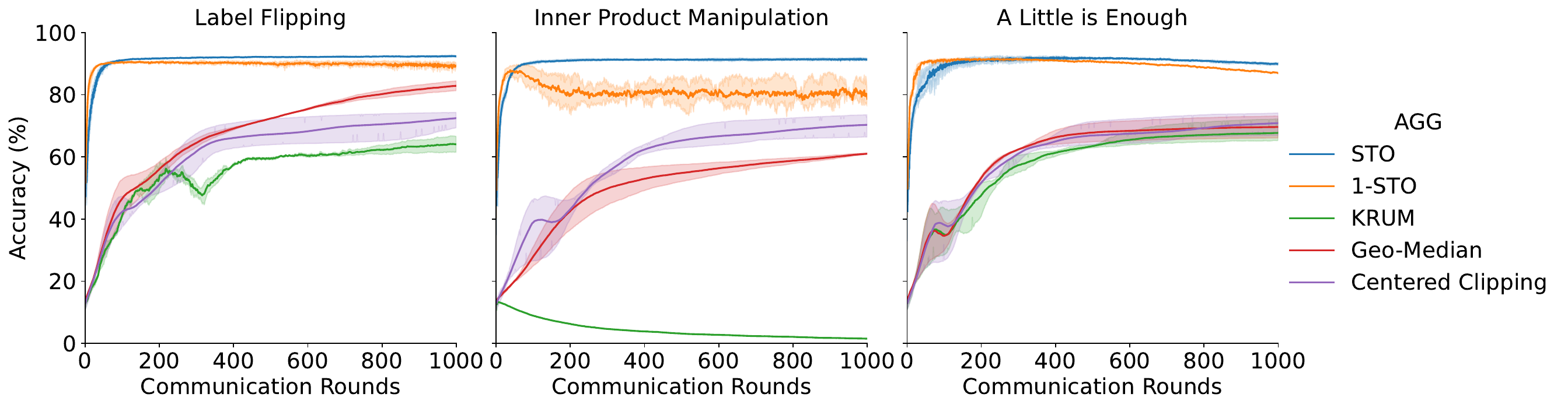}
        \caption{Comparisons with the baselines, where $1$-STO refers to $\beta$ stochastic sign compressor with $\beta=B$.}
    \label{fig:MNIST Byzantine comparisons}
\end{figure}

We compare Algorithm \ref{alg: alg 2} on MLP via MNIST and Dirichlet distribution with some renowned Byzantine resilient algorithms from literature, specifically, with Krum \cite{blanchard2017machine}, geometric median \cite{chen2017distributed}, centered clipping \cite{karimireddy2021learning} under three adversary models, including label flipping, inner product manipulation \cite{xie2020fall}, the "A little is enough" \cite{baruch2019little}. Following \cite{karimireddy2021learning}, $\tau$ is set to be 10 in centered clipping since momentum is switched off. The adversaries are described in Appendix \ref{app: byzantine} with details. We consider $20$ mobile Byzantine clients and allow adversaries to manipulate the mini-batch stochastic gradient but assume an honest compressor that will send out the correctly compressed corrupted messages to PS.

Throughout the experiments, it is observed that our $\beta$ stochastic sign compressor beats all the other baseline algorithms with or without privacy protection, i.e., when $\beta=0$ or $\beta=B$. Notably, our compressor transmits merely one-bit message per gradient coordinate.


\bibliography{ms}
\bibliographystyle{alpha}
\newpage
\appendix
\section{Comparisons with \cite{jin2020stochastic}.}
\label{sec: comparison with jin}
We list some of the key elements in Table \ref{tab:comparison with jin}. 
\begin{table}[!htb]
\caption{Point-by-point comparisons with \cite{jin2020stochastic}}
\label{tab:comparison with jin}
\centering
\begin{tblr}{width=\linewidth,colspec={X[0.5,l,m] X[0.5,l,m] X[1,l,m] X[1,l,m]},hspan=even,}
\hline
                                & \SetCell[c=1]{c}{} & \SetCell[c=1]{c}{\cite{jin2020stochastic}}                                                 & 
                                \SetCell[c=1]{c}{Algorithm \ref{alg: alg 2}} \\ \hline
\SetCell[r=2]{c}{Gradient}       & Batch-size           & Full-batch (True Gradient) {[}Eq. (33){]}                                                                     &\SetCell[c=1]{c}{Mini-batch Stochastic Gradient}\\
                                & Distribution         & Bounded {[}Theorem 3{]}                                                                                       &\SetCell[c=1]{c}{ A variety of distributions from bounded to unbounded in Assumptions  \ref{ass: BG condition}, \ref{ass: 4 Bounded gradient}, \ref{ass: sub gaussianity1}, \ref{ass: heavy tailed}.} \\ \hline
\SetCell[r=3]{c}{\hfill Residual} & Weak Signal Strength &$\Delta(M)$ as a function of client number $M$ without explicit form {[}Theorem 2 and Remark 3{]}. &$\delta(M):=\frac{2(B+\beta)}{p}\sqrt{\frac{2}{M}\log\frac{6}{3-5c}}+2\sigma\sqrt{\frac{2\log\frac{6}{c}}{{Mn}}}$ for $c\in(0,\frac{3}{5}).$                                                                                                                                                                            \\
                                & Byzantine            & Implicit forms without quantitative results {[}Theorem 7 and Remark 6{]}.               & $\Upsilon$ terms in Theorem \ref{thm: sub-gaussian}, depending on the type of adversaries.                                                                                                                                              \\ & Noisy Gradient Tail      & No, since only bounded gradients are considered.                                                          & $\Xi$ terms in Theorem \ref{thm: sub-gaussian}, depending on the tail distribution.\\ \hline
\SetCell[c=2]{c}{Differential Privacy} & & No for sto-sign compressor; flawed arguments for DP-sign compressor.&  Theorem \ref{thm: sto sign dp necessity of beta}, Theorem \ref{thm: sto sign dp}, and Corollary \ref{corollary: constant clipping with d.p.}.\\ \hline
\SetCell[c=2]{c}{Partial Client Participation}    &   & No.                                                                                                          & Theoretically and empirically verified, and build adaptive Byzantine adversaries on it.                                                                                                         \\ \hline
\end{tblr}
\end{table}

\section{Additional Analysis}
\subsection{Definitions}
\label{app: definitions}

\begin{definition}
The coordinate-wise majority vote aggregation rule, denoted by $\agg_{\text{maj}}$, aggregates each coordinate $i\in [d]$ as follows: If there are more $1$ than $-1$ in $\sth{\hat{u}_{mi}, ~ m\in \calS}$, then $\agg_{\text{maj},i}\pth{\sth{\hat{u}_m, ~ m\in \calS}}$ outputs $1$. If there are more $-1$ than $1$ in $\sth{\hat{u}_{mi}, ~ m\in \calS}$, then $\agg_{\text{maj},i}\pth{\sth{\hat{u}_m, ~ m\in \calS}}$ outputs $-1$. Otherwise, $\agg_{\text{median},i}\pth{\sth{\hat{u}_m, ~ m\in \calS}}$ outputs $0$.  
\end{definition}

\subsection{Alternative Assumptions}
\label{sec: additional assumptions}
The following two alternative assumptions on the randomness of stochastic gradients are of decreasing levels of stringency.  
\begin{assumption}[Boundedness]
\label{ass: Bounded stochastic gradient}
The $\ell_{\infty}$ norm of all possible stochastic gradients is upper bounded.
Formally, let $m\in [M]$ be an arbitrary client and $\bm{g}$ be an arbitrary stochastic gradient that client $m$ obtains. For any coordinate $i\in [d]$, there exists $\tilde{B}_i>0$ such that $\abth{g_{i}}\le \tilde{B}_i$. 
Let $\tilde{B} = \max_{i\in [d]}\tilde{B}_i$.
\end{assumption}

The following alternative assumption relaxes the boundedness requirement, and allows the stochastic gradients to be supported over the entire $\reals^d$.  
\begin{assumption}[Gaussianity]
\label{ass: gaussianity}
For a given client $m\in [M]$, at any query $w\in\mathbb{R}^d$, the stochastic gradient $\bm{g}_m(w)$
is an independent unbiased estimate of $\nabla f_{m}(w)$ that is coordinate-wise related to the gradient $\nabla f_m(w)$ as $\bm{g}_{mi}(w)=\nabla f_{mi}(w)+\bm{\xi}_{mi} ~ \forall\, i\in [d],$
where $\bm{\xi}_{mi}\sim\calN\pth{0,\sigma_{mi}^2}$. Let $\sigma^2 : = \max_{m\in [M], i\in [d]} \sigma^2_{mi}$. 
\end{assumption}
\subsection{Refined Privacy Preservation}
\label{app: privacy}
\begin{definition}
\label{def: distance}
For any given $B>0$, let $\calC_B := (-\infty, -B)\cup (B, \infty)$.
For each $g\in \reals$, define { $\dist\pth{g, \calC_B} : = \inf_{g^{\prime}\in \calC_B} \abth{g-g^{\prime}}$}.
\end{definition}
{
\begin{theorem}
\label{thm: sto sign dp}
 Let $\bm{g}, \bm{g}^{\prime}\in \calG\subseteq \reals^d$ be an arbitrary pair of gradient inputs such that $\bm{g}^{\prime}\not=\bm{g}$. $\calM_{B,\beta}$ 
is $\max_{\bm{g}\in \calG}\sum_{i=1}^d \log\pth{1 +  \frac{1}{\beta + \dist\pth{g_i, \calC_B}}}$-DP on $\calG$ for $\beta>0$. 
\end{theorem}
}
\begin{corollary}
\label{corollary: constant clipping with d.p.}
    Given the same definitions as in Theorem \ref{thm: sto sign dp}, $\calM_{B,\beta}$ 
is $\pth{\frac{1}{\beta}}$-DP.
\end{corollary}

\subsection{Alternative Convergence Rates}
\label{sec: additional convergence rates}
\begin{corollary}
\label{corollary: gaussian}
Suppose that Assumptions \ref{ass: 4 Bounded gradient} and \ref{ass: gaussianity} hold. 
Choose $B>B_0+\sigma$. Fix $t\ge1$ and $i\in[d]$. 
Let $c>0$ be any given constant such that $c< \frac{3}{5}$. 

When the system adversary is adaptive or { when} the system adversary is static but with $\tau(t)\le \frac{2}{p^2}\log\frac{6}{c}$, if $\abth{{ \nabla F}(w(t))_i} \ge \frac{2(B+\beta)}{pM}\tau(t)+ \frac{B+\beta}{2\sqrt{2\pi}}\exp\pth{-\frac{n}{2}}+ \frac{2(B+\beta)}{p}\sqrt{\frac{2}{M}\log\frac{6}{3-5c}}+2\sigma\sqrt{\frac{2\log\frac{6}{c}}{{Mn}}}$, then Eq. \,\eqref{eq: error bound} holds.

When the system adversary is static with $\tau(t) >\frac{2}{p^2}\log\frac{6}{c}$, if $\abth{{ \nabla F}(w(t))_i} \ge \frac{3(B+\beta)\tau(t)}{M}+ \frac{B+\beta}{2\sqrt{2\pi}}\exp\pth{-\frac{n}{2}}+ \frac{2(B+\beta)}{p}\sqrt{\frac{2}{M}\log\frac{6}{3-5c}}+2\sigma\sqrt{\frac{2\log\frac{6}{c}}{{Mn}}}$, then Eq.\,\eqref{eq: error bound} holds.

\end{corollary}


\begin{corollary}
\label{corollary: convergence gaussian}
Suppose Assumptions \ref{ass: 1 lower bound}, \ref{ass: 2 smmothness}, \ref{ass: 4 Bounded gradient}, and \ref{ass: gaussianity} hold. For any given $T$, $B=(1+\epsilon_0)B_0$ for $\epsilon_0>\frac{\sigma}{B_0}$,  and $c$ such that $0<c<\frac{3}{5}$. Recall that $R$ is the random time.

When the system adversary is adaptive or { when} the system adversary is static but with $\tau(t)\le \frac{2}{p^2}\log\frac{6}{c}$, we have
\begin{align*}
    \expect{\|\nabla F(w(R))\|_1}\le \frac{1}{c}\left[
    \frac{F(w(0)) - F^*}{\sum_{t=0}^{T-1}\eta_t} + \frac{Ld\sum_{t=0}^{T-1}\eta_t^2}{2\sum_{t=0}^{T-1}\eta_t} + \frac{4 d(B+\beta)}{p}\sqrt{\frac{2}{M}\log\frac{6}{3-5c}}\right.&\\
    \left.+4\sigma d\sqrt{\frac{2\log\frac{6}{c}}{{Mn}}} + \frac{d}{\sqrt{2\pi}}(B+\beta)\exp\pth{-\frac{n}{2}}+4d\frac{(B+\beta)\sum_{t=0}^{T-1}\eta_t\tau(t)}{pM\sum_{t=0}^{T-1}\eta_t}\right]&
\end{align*}

On the other hand, when the system adversary is static with $\tau(t) >\frac{2}{p^2}\log\frac{6}{c}$, we have
\begin{align*}
    \expect{\|\nabla F(w(R))\|_1}\le \frac{1}{c}\left[
    \frac{F(w(0)) - F^*}{\sum_{t=0}^{T-1}\eta_t} + \frac{Ld\sum_{t=0}^{T-1}\eta_t^2}{2\sum_{t=0}^{T-1}\eta_t} + \frac{4 d(B+\beta)}{p}\sqrt{\frac{2}{M}\log\frac{6}{3-5c}}\right.&\\
    \left.+4\sigma d\sqrt{\frac{2\log\frac{6}{c}}{{Mn}}} + \frac{d}{\sqrt{2\pi}}(B+\beta)\exp\pth{-\frac{n}{2}}+6d\frac{(B+\beta)\sum_{t=0}^{T-1}\eta_t\tau(t)}{M\sum_{t=0}^{T-1}\eta_t}\right]&
\end{align*}
\end{corollary}

\begin{corollary}
\label{corollary: gaussian learning rate}
Suppose Assumptions \ref{ass: 1 lower bound}, \ref{ass: 2 smmothness}, \ref{ass: 4 Bounded gradient}, and \ref{ass: gaussianity} hold. For any given $T$, $B=(1+\epsilon_0)B_0$ for $\epsilon_0>\frac{\sigma}{B_0}$,  and $c$ such that $0<c<\frac{3}{5}$. Recall that $R$ is the random time.
\begin{itemize}
    \item Set the learning rate as $\eta_t=\frac{1}{\sqrt{dT}},$
    when the system adversary is adaptive or { when} the system adversary is static but with $\tau(t)\le \frac{2}{p^2}\log\frac{6}{c}$, we have
    \begin{align*}
         \expect{\|\nabla F(w(R))\|_1}\le\frac{1}{c}\left[\frac{\pth{F(w(0))-F^*}\sqrt{d}}{\sqrt{T}}+\frac{L\sqrt{d}}{2\sqrt{T}}+\frac{d}{\sqrt{2\pi}}(B+\beta)\exp\pth{-\frac{n}{2}}\right.\\
         \left.+\frac{4d(B+\beta)}{p}\sqrt{\frac{2}{M}\log\frac{6}{3-5c}}+4d\sigma\sqrt{\frac{2\log\frac{6}{c}}{{Mn}}}+4d\frac{(B+\beta)\sum_{t=0}^{T-1}\tau(t)}{pTM}\right].
    \end{align*}

On the other hand, when the system adversary is static with $\tau(t) >\frac{2}{p^2}\log\frac{6}{c}$, we have
\begin{align*}
    \expect{\|\nabla F(w(R))\|_1}\le\frac{1}{c}\left[\frac{\pth{F(w(0))-F^*}\sqrt{d}}{\sqrt{T}}+\frac{L\sqrt{d}}{2\sqrt{T}}+\frac{d}{\sqrt{2\pi}}(B+\beta)\exp\pth{-\frac{n}{2}}\right.\\\left.+\frac{4d(B+\beta)}{p}\sqrt{\frac{2}{M}\log\frac{6}{3-5c}}+4d\sigma\sqrt{\frac{2\log\frac{6}{c}}{{Mn}}}+6d\frac{(B+\beta)\sum_{t=0}^{T-1}\tau(t)}{TM}\right].
\end{align*}

    \item Set the learning rate as $\eta_t=\frac{1}{\sqrt{d(t+1)}},$ 
    when the system adversary is adaptive or { when} the system adversary is static but with $\tau(t)\le \frac{2}{p^2}\log\frac{6}{c}$, we have
    \begin{align*}
         \expect{\|\nabla F(w(R))\|_1}\le O \left(\frac{1}{c}\left[\frac{\pth{F(w(0))-F^*}\sqrt{d}}{\sqrt{T}}+\frac{L\sqrt{d}\log T}{2\sqrt{T}}+\frac{d}{\sqrt{2\pi}}(B+\beta)\exp\pth{-\frac{n}{2}}\right.\right.\\
         \left.\left.+\frac{4d(B+\beta)}{p}\sqrt{\frac{2}{M}\log\frac{6}{3-5c}}+4d\sigma\sqrt{\frac{2\log\frac{6}{c}}{{Mn}}}+4d\frac{(B+\beta)\max_{t\in[T-1]}\tau(t)}{pM}\right]\right).
    \end{align*}

On the other hand, when the system adversary is static with $\tau(t) >\frac{2}{p^2}\log\frac{6}{c}$, we have
\begin{align*}
    \expect{\|\nabla F(w(R))\|_1}\le O\left(\frac{1}{c}\left[\frac{\pth{F(w(0))-F^*}\sqrt{d}}{\sqrt{T}}+\frac{L\sqrt{d}\log T}{2\sqrt{T}}+\frac{d}{\sqrt{2\pi}}(B+\beta)\exp\pth{-\frac{n}{2}}\right.\right.\\\left.\left.+\frac{4d(B+\beta)}{p}\sqrt{\frac{2}{M}\log\frac{6}{3-5c}}+4d\sigma\sqrt{\frac{2\log\frac{6}{c}}{{Mn}}}+6d\frac{(B+\beta)\max_{t\in[T-1]}\tau(t)}{M}\right]\right).
\end{align*}
    
\end{itemize}

\end{corollary}

\begin{corollary}
\label{corollary: sampling probability with all byzantine}
Suppose that Assumption \ref{ass: Bounded stochastic gradient} holds. 
Choose $B=\tilde{B}$.  Fix $t\ge1$ and $i\in[d]$. Let $c$ be any given positive constant such that $c< \frac{3}{5}$. 

When the system adversary is adaptive or when the system adversary is static but with $\tau(t)\le \frac{2}{p^2}\log\frac{6}{c}$, if $\abth{{ \nabla F_i}(w(t))} \ge \frac{2(B+\beta)}{pM}\tau(t)+ \frac{2(B+\beta)}{p}\sqrt{\frac{2}{M}\log\frac{6}{3-5c}}+2\sigma\sqrt{\frac{2\log\frac{6}{c}}{{Mn}}}$, then Eq.\,\eqref{eq: error bound} holds.

When the system adversary is static with $\tau(t) >\frac{2}{p^2}\log\frac{6}{c}$, if $\abth{{ \nabla F_i}(w(t))} \ge \frac{3(B+\beta)}{M}\tau(t)+ \frac{B+\beta}{2\sqrt{2\pi}}\exp\pth{-\frac{n}{2}}+ \frac{2(B+\beta)}{p}\sqrt{\frac{2}{M}\log\frac{6}{3-5c}}+2\sigma\sqrt{\frac{2\log\frac{6}{c}}{{Mn}}}$, then Eq.\,\eqref{eq: error bound} holds. 
\end{corollary}

\begin{corollary}
\label{corollary: convergence all Byzantine}
Suppose Assumptions \ref{ass: 1 lower bound}, \ref{ass: 2 smmothness}, and \ref{ass: Bounded stochastic gradient} hold. For any given $T$ and $c$ such that $0<c<\frac{3}{5}$. 

When the system adversary is adaptive or { when} the system adversary is static but with $\tau(t)\le \frac{2}{p^2}\log\frac{6}{c}$, we have
\begin{align*}
    \expect{\|\nabla F(w(R))\|_1}\le \frac{1}{c}\left[
    \frac{F(w(0)) - F^*}{\sum_{t=0}^{T-1}\eta_t} + \frac{Ld\sum_{t=0}^{T-1}\eta_t^2}{2\sum_{t=0}^{T-1}\eta_t} + \frac{4 d(B+\beta)}{p}\sqrt{\frac{2}{M}\log\frac{6}{3-5c}}\right.&\\
    \left.+4\sigma d\sqrt{\frac{2\log\frac{6}{c}}{{Mn}}} +4d\frac{(B+\beta)\sum_{t=0}^{T-1}\eta_t\tau(t)}{pM\sum_{t=0}^{T-1}\eta_t}\right]&
\end{align*}

On the other hand, when the system adversary is static with $\tau(t) >\frac{2}{p^2}\log\frac{6}{c}$, we have
\begin{align*}
    \expect{\|\nabla F(w(R))\|_1}\le \frac{1}{c}\left[
    \frac{F(w(0)) - F^*}{\sum_{t=0}^{T-1}\eta_t} + \frac{Ld\sum_{t=0}^{T-1}\eta_t^2}{2\sum_{t=0}^{T-1}\eta_t} + \frac{4 d(B+\beta)}{p}\sqrt{\frac{2}{M}\log\frac{6}{3-5c}}\right.&\\
    \left.+4\sigma d\sqrt{\frac{2\log\frac{6}{c}}{{Mn}}} +4d\frac{(B+\beta)\sum_{t=0}^{T-1}\eta_t\tau(t)}{pM\sum_{t=0}^{T-1}\eta_t}\right]&
\end{align*}
\end{corollary}
\begin{corollary}
\label{corollary: bounded learning rate}
Suppose Assumptions \ref{ass: 1 lower bound}, \ref{ass: 2 smmothness}, and \ref{ass: Bounded stochastic gradient} hold. For any given $T$, $B=(1+\epsilon_0)B_0$ for $\epsilon_0>\frac{\sigma}{B_0}$,  and $c$ such that $0<c<\frac{3}{5}$. Recall that $R$ is the random time.
\begin{itemize}
    \item Set the learning rate as $\eta_t=\frac{1}{\sqrt{dT}},$
    when the system adversary is adaptive or { when} the system adversary is static but with $\tau(t)\le \frac{2}{p^2}\log\frac{6}{c}$, we have
    \begin{align*}
         \expect{\|\nabla F(w(R))\|_1}\le\frac{1}{c}\left[\frac{\pth{F(w(0))-F^*}\sqrt{d}}{\sqrt{T}}+\frac{L\sqrt{d}}{2\sqrt{T}}+\frac{4d(B+\beta)}{p}\sqrt{\frac{2}{M}\log\frac{6}{3-5c}}\right.\\
         \left.+4d\sigma\sqrt{\frac{2\log\frac{6}{c}}{{Mn}}}+4d\frac{(B+\beta)\sum_{t=0}^{T-1}\tau(t)}{pTM}\right].
    \end{align*}

On the other hand, when the system adversary is static with $\tau(t) >\frac{2}{p^2}\log\frac{6}{c}$, we have
\begin{align*}
         \expect{\|\nabla F(w(R))\|_1}\le\frac{1}{c}\left[\frac{\pth{F(w(0))-F^*}\sqrt{d}}{\sqrt{T}}+\frac{L\sqrt{d}}{2\sqrt{T}}+\frac{4d(B+\beta)}{p}\sqrt{\frac{2}{M}\log\frac{6}{3-5c}}\right.\\
         \left.+4d\sigma\sqrt{\frac{2\log\frac{6}{c}}{{Mn}}}+6d\frac{(B+\beta)\sum_{t=0}^{T-1}\tau(t)}{TM}\right].
    \end{align*}

    \item Set the learning rate as $\eta_t=\frac{1}{\sqrt{d(t+1)}},$ 
    when the system adversary is adaptive or { when} the system adversary is static but with $\tau(t)\le \frac{2}{p^2}\log\frac{6}{c}$, we have
    \begin{align*}
         \expect{\|\nabla F(w(R))\|_1}\le O \left(\frac{1}{c}\left[\frac{\pth{F(w(0))-F^*}\sqrt{d}}{\sqrt{T}}+\frac{L\sqrt{d}\log T}{2\sqrt{T}}+\frac{4d(B+\beta)}{p}\sqrt{\frac{2}{M}\log\frac{6}{3-5c}}\right.\right.\\
         \left.\left.+4d\sigma\sqrt{\frac{2\log\frac{6}{c}}{{Mn}}}+4d\frac{(B+\beta)\max_{t\in[T-1]}\tau(t)}{pM}\right]\right).
    \end{align*}

On the other hand, when the system adversary is static with $\tau(t) >\frac{2}{p^2}\log\frac{6}{c}$, we have
\begin{align*}
    \expect{\|\nabla F(w(R))\|_1}\le O\left(\frac{1}{c}\left[\frac{\pth{F(w(0))-F^*}\sqrt{d}}{\sqrt{T}}+\frac{L\sqrt{d}\log T}{2\sqrt{T}}+\frac{4d(B+\beta)}{p}\sqrt{\frac{2}{M}\log\frac{6}{3-5c}}\right.\right.\\\left.\left.+4d\sigma\sqrt{\frac{2\log\frac{6}{c}}{{Mn}}}+6d\frac{(B+\beta)\max_{t\in[T-1]}\tau(t)}{M}\right]\right).
\end{align*}
\end{itemize}
\end{corollary}
\section{Proofs}\label{sec:Proofs}
\subsection{Algorithm \ref{alg: alg 1}}
\begin{proof}[\bf Proof of Proposition \ref{prop: iteration perturbation: l infinity}]
We observe that 
\begin{align*}
\linf{\nabla F_{t+1}}&\le \linf{\nabla F_{t}}+\linf{\nabla F_{t+1}-\nabla F_t}\\
&\le \linf{\nabla F_{t}}+\norm{\nabla F_{t+1}-\nabla F_t}\\
&\le \linf{\nabla F_{t}}+L\eta_t\sqrt{d}\\
& = \linf{\nabla F_{t}} + \frac{c}{\sqrt{t+1}}. 
\end{align*}
Similarly, we have $\linf{\nabla F_{t+1}}\ge  \linf{\nabla F_{t}}-\frac{c}{\sqrt{t+1}}.$

Simple algebra leads to the conclusion in Proposition \ref{prop: iteration perturbation: l infinity}. 
\end{proof}
\begin{proof}[\bf Proof of Theorem \ref{thm: level bound l infinty}]
We try proof by induction. 
We first show that base case. 
In Algorithm \ref{alg: alg 1}, we have 
\[
\max_{m\in [M]} \|\nabla F_m(0)\|_{\infty} \le B(1) \le 2 \max_{m\in [M]} \|\nabla F_m(0)\|_{\infty}.   
\]  
A simple rewrite gives $\frac{1}{2} B(1) \le \max_{m\in [M]} \|\nabla F_m(0)\|_{\infty}\le B(1).$

\paragraph{On $t=2$, i.e., the second iteration.} 
In this iteration, gradients $\{\nabla F_m(1)\}$ are involved. 

We need to figure out the relation between $\{\nabla F_m(1)\}$ and $\{\nabla F_m(0)\}$. 
At round 2, the following is true: 
\begin{itemize} 
\item When $\max_{m\in [M]}\|\nabla F_m(1)\|_{\infty} \le \frac{5 c}{\sqrt{2}}$, it is true that $B(2) = \frac{4 c}{\sqrt{2}}$. 
\item When $\max_{m\in [M]}\|\nabla F_m(1)\|_{\infty} > \frac{5 c}{\sqrt{2}}$. Since $\|\nabla F_m(1)\|_{\infty} \le \|\nabla F_m(0)\|_{\infty} + c$, it must be true that 
\[
\|\nabla F_m(0)\|_{\infty} \ge 2c. 
\]

\begin{itemize}
\item If ``level up'' occurs, it must be true that $\max_{m\in [M]}\|\nabla F_m(1)\|_{\infty}>B(1)$, and that $B(2) = 2B(1)$. Then 
\[
\max_{m\in [M]}\|\nabla F_m(1)\|_{\infty} < \frac{3}{2} \max_{m\in [M]} \|\nabla F_m(0)\|_{\infty}\le \frac{3}{2} B(1) \le 2B(1) = B(2). 
\]
Hence, 
\[
\frac{1}{2}B(2)\le \max_{m\in [M]}\|\nabla F_m(1)\|_{\infty} \le  B(2). 
\]
\item If ``level down'' occurs, it must be true that $\|\nabla F_m(1)\|_{\infty}\le \frac{1}{2}B(1)$, and that $B(2) =\frac{1}{2}B(1)$. Moreover, 
\begin{align*}
\max_{m\in [M]}\|\nabla F_m(1)\|_{\infty} > \frac{1}{2} \max_{m\in [M]} \|\nabla F_m(0)\|_{\infty} > \frac{1}{2} \frac{1}{2} B(1) = \frac{1}{2} B(2). 
\end{align*}
Hence, 
\[
\frac{1}{2} B(2) < \max_{m\in [M]}\|\nabla F_m(1)\|_{\infty} \le B(2). 
\]
\item If no update of $B$ occurs, then $B(2) = B(1)$, by the leveling rule, we know that 
\[
\frac{1}{2}B(2) = \frac{1}{2}B(1) < \max_{m\in [M]}\|\nabla F_m(1)\|_{\infty}\le B(1) = B(2). 
\]
\end{itemize}
\end{itemize}

Therefore, it is true that 
\begin{align*}
\frac{1}{2}B(2) \indc{\max_{m\in [M]}\|\nabla F_m (1) \|_{\infty}\ge \frac{5c}{\sqrt{2}}} 
 \le \max_{m\in [M]}\|\nabla F_m(1)\|_{\infty} 
\le B(2) .    
\end{align*}

We assume the above invariance holds up to iteration $t+1$. 
Next we consider the $t+2$-th iteration. 
We have 
\begin{itemize}
\item If $\max_{m\in [M]}\|\nabla F_m(t+1)\|_{\infty} \le \frac{5 c}{\sqrt{t+2}}$, then $B(t+2) = \frac{5 c}{\sqrt{t+2}}$. 
That is, $\max_{m\in [M]}\|\nabla F_m(t+1)\|_{\infty}\le B(t+2)$. 
\item Otherwise, $\max_{m\in [M]}\|\nabla F_m(t+1)\|_{\infty} > \frac{5 c}{\sqrt{t+2}}$. 
Since $\|\nabla F_m(t+1)\|_{\infty} \le\|\nabla F_m(t)\|_{\infty} + \frac{c}{\sqrt{t+1}}$, it must be true that 
\[
\norm{\nabla F_m(t)} \ge \frac{2c}{\sqrt{t+1}}. 
\]
Thus, by Proposition \ref{prop: iteration perturbation: l infinity}, we have 
\[
\frac{1}{2} \max_{m\in [M]}\|\nabla F_m(t)\|_{\infty} < \max_{m\in [M]}\|\nabla F_m(t+1)\|_{\infty} < \frac{3}{2} \max_{m\in [M]}\|\nabla F_m(t)\|_{\infty}. 
\]
\begin{itemize}
\item If ``level up'' occurs, it must be true that $\max_{m\in [M]}\|\nabla F_m(t+1)\|_{\infty}>B(t+1)$. Since this is a ``level up'', it is true that $B(t+2) = 2B(t+1)$. So we have 
\[
\max_{m\in [M]}\|\nabla F_m(t+1)\|_{\infty} < 2 \max_{m\in [M]}\|\nabla F_m(t)\|_{\infty}\le 2B(t+1) = B(t+2). 
\]
\item If ``level down'' occurs, it must be true that $\|\nabla F_m(t+1)\|_{\infty}\le \frac{1}{2}B(t+1)$ for all $m$, and that $B(t+2) =\frac{1}{2}B(t+1)$. Moreover, 
\begin{align*}
\max_{m\in [M]}\|\nabla F_m(t+1)\|_{\infty} &> \frac{1}{2} \max_{m\in [M]}\|\nabla F_m(t)\|_{\infty} \\
&> \frac{1}{2} \frac{1}{2} B(t+1) \indc{\max_{m\in [M]}\|\nabla F_m (t)\|_{\infty} \ge \frac{5c}{\sqrt{t+1}}}, 
\end{align*}
where the last inequality follows from the induction hypothesis. 

Suppose that $\max_{m\in [M]}\|\nabla F_m (t)\|_{\infty} < \frac{5c}{\sqrt{t+1}}$. It must be true that $B(t+1) = \frac{5c}{\sqrt{t+1}}$. Since we have a ``level-down'', 
    \begin{align*}
     \|\nabla F_m(t+1)\|_{\infty}\le \frac{1}{2}B(t+1) = \frac{1}{2} \frac{5c}{\sqrt{t+1}} < \frac{5c}{\sqrt{t+2}}, 
    \end{align*}
    which contradicts the case assumption that $\max_{m\in [M]}\norm{\nabla F_m(t+1)} > \frac{5c}{\sqrt{t+2}}$. Thus, 
    \[
    \indc{\max_{m\in [M]}\|\nabla F_m (t)\|_{\infty} \ge \frac{5c}{\sqrt{t+1}}} = 1 = \indc{\max_{m\in [M]}\|\nabla F_m (t+1)\|_{\infty} \ge \frac{5c}{\sqrt{t+2}}}
    \]
\item If no update of $B$ occurs, then $B(t+2) = B(t+1)$, by the leveling rule, we know that 
\[
\frac{1}{2}B(t+2) = \frac{1}{2}B(t+1) < \max_{m\in [M]}\|\nabla F_m(t+1)\|_{\infty} \le B(t+1) = B(t+2). 
\]
\end{itemize}
\end{itemize}

Therefore, the proof of the induction is complete. 
\end{proof}

\begin{proof}[\bf Proof of Theorem \ref{thm: prob of sign disagreement}]
    By symmetry, without loss of generality, let us assume $\sign\pth{\nabla F(w(t))_i}=-1$. 
Define $X_{mi}=\indc{\sign\pth{\hat{g}_{mi}}\neq -1}$.
For ease of exposition, we drop the round $t$ for simplicity purposes.
Next,
$$\prob{\frac{1}{M}\sum_{m=1}^M \sign\pth{\hat{g}_{mi}(t)}\neq -1}\le \prob{\sum_{m=1}^M X_{mi}\ge \frac{M}{2}}.$$
By the gradient compression rule, we know that $\expect{X_{mi}}=\frac{1}{2}+\frac{\bm{g}_{mi}}{2B_{t+1}}$

It follows that
\begin{align*}
    \sum_{i=1}^M X_{mi}-\expect{\sum_{i=1}^M X_{mi}}&\ge \frac{M}{2}-\sum_{m=1}^M\pth{\frac{1}{2}+\frac{\bm{g}_{mi}}{2B_{t+1}}}=-\sum_{m=1}^M\frac{\bm{g}_{mi}}{2B_{t+1}}>0.
\end{align*}

When $\abth{\nabla F_i(t)}\ge c_1B_t\sqrt{\frac{2}{M}}$, by Hoeffding's,
\begin{align*}
    \prob{\sum_{i=1}^M X_{mi}-\expect{\sum_{i=1}^M X_{mi}}\ge -\sum_{m=1}^M\frac{\bm{g}_{mi}}{2B_t}}& \le \prob{\sum_{i=1}^M X_{mi}-\expect{\sum_{i=1}^M X_{mi}}\ge c_1\sqrt{\frac{M}{2}}}\\
    & \le \exp\pth{-c_1^2}\\
    & = \frac{1-c_0}{2},
\end{align*}
where the last follows from the fact $c_1=\sqrt{\log\pth{\frac{2}{1-c_0}}}.$
\end{proof}
%
\begin{proof}[\bf Proof of Theorem \ref{thm: convergence alg 1}]
Assumption \ref{ass: 2 smmothness} gives us
\begin{align*}
    F(t+1)-F(t) & \le \langle\nabla F(t), w_{t+1}-w_t\rangle + \frac{L}{2}\norm{w_{t+1}-w_t}^2\\
    & = -\eta_t\lnorm{\nabla F(t)}{1} + 2\eta_t \sum_{i=1}^d \abth{\nabla F_i (t)}\indc{\hat{g}_i\neq\sign\pth{\nabla F_i(t)}} + \frac{Ld}{2}\eta_t^2.
\end{align*}
Conditional on $w_t$, we get
\begin{align*}
&\expect{F(t+1) - F(t)\mid w_t} \le -\eta_t\lnorm{\nabla F(t)}{1} + 2\eta_t \sum_{i=1}^d \abth{\nabla F_i (t)}\prob{\hat{g}_i\neq\sign\pth{\nabla F_i(t)}} + \frac{Ld}{2}\eta_t^2\\
& =     -\eta_t\lnorm{\nabla F(t)}{1} + \frac{Ld}{2}\eta_t^2   + 2\eta_t \sum_{i=1}^d \abth{\nabla F_i (t)}\prob{\hat{g}_i\neq\sign\pth{\nabla F_i(t)}}\\
& \qquad \qquad \qquad \qquad \qquad  \times \pth{\bm{1}\sth{|\nabla F_i(t)| \ge c_1B_{t+1}\sqrt{2/M}} + \bm{1}\sth{|\nabla F_i(t)| < c_1B_{t+1}\sqrt{2/M}}}  \\
& \overset{(a)}{\le}  -\eta_t\lnorm{\nabla F(t)}{1} + \frac{Ld}{2}\eta_t^2 + (1-c_0) \eta_t \lnorm{\nabla F(t)}{1} + 2c_1\eta_t d B_{t+1}\sqrt{2/M} \\
& = -c_0\eta_t\lnorm{\nabla F(t)}{1}+ \frac{Ld}{2}\eta_t^2 + 2c_1\eta_t d B_{t+1}\sqrt{2/M} \\
&\le -\eta_t c_0 \lnorm{\nabla F(t)}{1} + 2c_1 B_{t+1}d \eta_t \sqrt{\frac{2}{M}}\qth{\indc{\max\limits_{m\in[M]}\norm{\nabla F_m(t)}\ge\frac{5c}{\sqrt{t+1}}}+\indc{\max\limits_{m\in[M]}\norm{\nabla F_m(t)}<\frac{5c}{\sqrt{t+1}}}}  + \frac{Ld}{2}\eta_t^2\\
&\overset{(b)}{\le} -\eta_t c_1 \lnorm{\nabla F(t)}{1} +4c_1 d \eta_t \sqrt{\frac{2}{M}}\max\limits_{m\in[M]} \norm{\nabla F_m(t)}
+ 10c_1Ld\sqrt{\frac{2d}{M}}\eta_t^2 + \frac{Ld}{2}\eta_t^2\\
&\le -\eta_t c_0 \lnorm{\nabla F(t)}{1} + 4c_1 d \eta_t \sqrt{\frac{2}{M}} \max\limits_{m\in[M]}\norm{\nabla F_m (t) - \nabla F(t)} + 4c_1 d \eta_t \sqrt{\frac{2}{M}}\norm{\nabla F(t)}+ \\
&\qquad \left(\frac{Ld}{2} + 10c_2Ld\sqrt{\frac{2d}{M}}\right)\eta_t^2 \\
&\overset{(c)}{\le} -\eta_t c_0 \lnorm{\nabla F(t)}{1} + 4c_1 d \eta_t \sqrt{\frac{2}{M}}(\tilde{B}+1)\norm{\nabla F(t)}+ 4c_1 d \eta_t \sqrt{\frac{2}{M}}\tilde{G} + \pth{\frac{Ld}{2}+10c_1Ld\sqrt{\frac{2d}{M}}}\eta_t^2 \\
&\le -\eta_t c_0 \norm{\nabla F(t)} + 4c_1 d \eta_t \sqrt{\frac{2}{M}}(\tilde{B}+1)\norm{\nabla F(t)}+ 4c_1 d \eta_t \sqrt{\frac{2}{M}}\tilde{G} + \pth{\frac{Ld}{2}+10c_2Ld\sqrt{\frac{2d}{M}}}\eta_t^2.
\end{align*}
where inequality (a) follows from Theorem \ref{thm: prob of sign disagreement}, 
inequality (b) follows from Theorem \ref{thm: level bound l infinty}, 
and inequality (c) follows from Assumption \ref{ass: BG condition}. 


It follows that
\begin{align*}
    \expect{F(t+1) - F(t)}&\le -\eta_t c \expect{\norm{\nabla F(t)}} + 4c_1 d \eta_t \sqrt{\frac{2}{M}}(\tilde{B}+1)\expect{\norm{\nabla F(t)}}+ 4c_1 d \eta_t \sqrt{\frac{2}{M}}\tilde{G} + \\
    &\quad \left(\frac{Ld}{2} + 6c_1Ld\sqrt{\frac{2d}{M}}\right)\eta_t^2 \\
\end{align*}
Throughout the trajectory,
\begin{align*}
    \qth{c_0-4c_1d\sqrt{\frac{2}{M}}\pth{\tilde{B}+1}}\sum_{t=0}^{T-1}\eta_t \expect{\norm{\nabla F(t)}} \le F(0) - F^* + 4c_2 d \sum_{t=0}^{T-1}\eta_t \sqrt{\frac{2}{M}}\tilde{G} +  \left(\frac{Ld}{2} + 6c_2Ld\sqrt{\frac{2d}{M}}\right)\sum_{t=0}^{T-1}\eta_t^2.
\end{align*}
Given 
$$c_0-4\sqrt{\ln\pth{\frac{2}{1-c_0}}}d\sqrt{\frac{2}{M}}\pth{\tilde{B}+1}=c_0-4c_1d\sqrt{\frac{2}{M}}\pth{\tilde{B}+1}\ge\frac{1}{2},$$ we have
\begin{align*}
    \sum_{t=0}^{T-1}\eta_t \expect{\norm{\nabla F(t)}} \le 2F(0) - 2F^* + 8c_1 d \eta_t \sqrt{\frac{2}{M}}\tilde{G} +  \left(Ld + 12c_1Ld\sqrt{\frac{2d}{M}}\right)\sum_{t=0}^{T-1}\eta_t^2,
\end{align*}
which yields
\begin{small}
\begin{flalign*}
   & \expect{\norm{\nabla F(R)}} \le O \pth{\frac{{2{\Delta}\sqrt{d}}}{L\sqrt{T}} + 8d \sqrt{\frac{2}{M}\log\pth{\frac{2}{1-c_0}}}\tilde{G} + \left(1 + 12\sqrt{\frac{2d}{M}\log\pth{\frac{2}{1-c_0}}}\right)\frac{\sqrt{d}\log T}{\sqrt{T}}}.&
\end{flalign*}
\end{small}
This follows because $\sum_{t=0}^{T-1}\eta_t=\Theta(\sqrt{T}),$ $\sum_{t=0}^{T-1}\eta^2_t = \Theta(\log T).$
\end{proof}

\subsubsection{Privacy Preservation in Algorithm \ref{alg: alg 2}}

\begin{theorem}\cite[Corollary 3.15]{dwork2014algorithmic}
\label{thm: composition}
Let $\calM_i:\reals^d \to \{\pm 1\}^d$ be an $\epsilon_i$-differentially private algorithm for $i\in [k]$. Then  
$\calM_{[k]}(x) := \pth{\calM_1(x), \cdots, \calM_k(x)}$ is $\sum_{i=1}^k \epsilon_i$-differentially private.
\end{theorem}

\begin{proof}[\bf Proof of Theorem \ref{thm: sto sign dp necessity of beta} (Necessity of $\beta$)]
We first consider the setting when $\beta=0$. 
Let 
$$
\calG = \{\bm{g}\in \reals^d: ~ \exists i ~ s.t. \min\{|\bm{g}_i-B|,  |\bm{g}_i+B|\}\le 1 \}.
$$
Let $\bm{g}\in \calG$.  Without loss of generality, let us assume that $|\bm{g}_1-B|\le 1$, where $\bm{g}_1$ is the first entry of $\bm{g}$. 
If $\bm{g}_1\ge B$, then there exists $\bm{g}^{\prime}\in \reals^d$ such that $\bm{g}^{\prime}\not=\bm{g}$, $\bm{g}_1^{\prime}\in (-B, B)$, and $\|\bm{g} -\bm{g}^{\prime}\|_1 \le 1$. Let $\hat{\bm{g}}_1$ and $\hat{\bm{g}^{\prime}}_1$ be the compressed values of $\bm{g}_1$ and $\bm{g}_1^{\prime}$ under our compressor $\beta$-StoSign. It holds that 
\begin{align*}
\frac{\prob{\hat{\bm{g}^{\prime}_1}=-1}}{\prob{\hat{\bm{g}}_{1}=-1}}
& =\frac{\frac{B-\clip\sth{\bm{g}_1^{\prime}, B}}{2B}}{\frac{B-\clip\sth{\bm{g}_1, B}}{2B}} = \frac{B-\clip\sth{\bm{g}_1^{\prime}, B}}{B-\clip\sth{\bm{g}_1, B}} =\frac{B-\clip\sth{\bm{g}_1^{\prime}, B}}{B- B} = \infty.
\end{align*}
If $\bm{g}_1\in (-B, B)$, then there exists $\bm{g}^{\prime}\in \reals^d$ such that $\bm{g}^{\prime}\not=\bm{g}$, $\bm{g}_1^{\prime}\ge  B$, and $\|\bm{g} -\bm{g}^{\prime}\|_1 \le 1$. We have \begin{align*}
\frac{\prob{\hat{\bm{g}_1}=-1}}{\prob{\hat{\bm{g}}^{\prime}_{1}=-1}}
& =\frac{\frac{B-\clip\sth{\bm{g}_1, B}}{2B}}{\frac{B-\clip\sth{\bm{g}_1^{\prime}, B}}{2B}} = \frac{B-\clip\sth{\bm{g}_1, B}}{B-\clip\sth{\bm{g}_1^{\prime}, B}} =\frac{B-\clip\sth{g_1, B}}{B- B} = \infty.
\end{align*}
Since a finite differential privacy quantification does not hold for any pair of gradients $\bm{g}$ and $\bm{g}^\prime$, no differential privacy implies as per Definition \ref{def: d.p. definition}, proving the first part of the theorem. 


When $\beta>0$, for any $\bm{g}, \bm{g}^{\prime}\in \reals^d$ such that $\bm{g}^{\prime}\not=\bm{g}$ and $\|\bm{g} -\bm{g}^{\prime}\|_1 \le 1$, and for each coordinate $i\in [d]$, it holds that 
%
\begin{align*}
\frac{\prob{\hat{\bm{g}}^{\prime}_i=-1}}{\prob{\hat{\bm{g}}_{i}=-1}} 
& = \frac{\frac{B+\beta-\clip\sth{\bm{g}_1^{\prime}, B}}{2B+2\beta}}{\frac{B+\beta-\clip\sth{\bm{g}_1, B}}{2B+2\beta}} 
= \frac{B+\beta-\clip\sth{\bm{g}_1^{\prime}, B}}{B+\beta-\clip\sth{\bm{g}_1, B}} \le \frac{2B+\beta}{\beta}. 
\end{align*}

Similarly, we can show the same upper bound for $\prob{\hat{\bm{g}}^{\prime}_i=1}/\prob{\hat{\bm{g}}_{i}=1}$. That is, for the $i$-th coordinate,  the compressor $\beta$-StoSign is coordinate-wise $\log\pth{\frac{2B+\beta}{\beta}}$- differentially private. 
By Theorem \ref{thm: composition}, we conclude that the compressor $\beta$-StoSign is $d\cdot\log\pth{\frac{2B+\beta}{\beta}}$- differentially private for the entire gradient.
\end{proof}


%

\begin{proof}[\bf Proof of Theorem \ref{thm: sto sign dp} (Smaller Collection of Gradients)]
Let $\bm{g}, \bm{g}^{\prime}\in \reals^d$ be an arbitrary pair of gradient inputs such that $\bm{g}^{\prime}\not=\bm{g}$ and $\|\bm{g} -\bm{g}^{\prime}\|_1 \le 1$.
For each coordinate $i\in [d]$, it holds that 
\begin{align}
\frac{\prob{\hat{\bm{g}}^{\prime}_i=-1}}{\prob{\hat{\bm{g}}_{i}=-1}}
& = \frac{\frac{B+\beta-\clip\sth{\bm{g}_i^{\prime}, B}}{2B+2\beta}}{\frac{B+\beta-\clip\sth{\bm{g}_i, B}}{2B+2\beta}} = \frac{B+\beta-\clip\sth{\bm{g}_i^{\prime}, B} }{B+\beta- \clip\sth{\bm{g}_i, B}} \notag \\
& = \frac{B+\beta-\clip\sth{\bm{g}_i, B}+ \clip\sth{\bm{g}_i, B}-\clip\sth{\bm{g}_i^{\prime}, B} }{B+\beta- \clip\sth{\bm{g}_i, B}} \notag\\
& \le 1+ \frac{\abth{\bm{g}_i-\bm{g}_i^{\prime}}}{B+\beta- \clip\sth{\bm{g}_i, B}}\label{eq: collection}\\
& \le 1 + \frac{1}{B+\beta- \clip\sth{\bm{g}_i, B}} \notag\\
& \le 1 +  \frac{1}{\beta + \dist\pth{\bm{g}_i, \calC_B}}. \notag
\end{align}
By Theorem \ref{thm: composition}, we conclude that the compressor $\beta$-StoSign is $\max_{\bm{g}\in \calG}\sum_{i=1}^d \log\pth{1 +  \frac{ 1}{\beta + \dist\pth{\bm{g}_i, \calC_B}}}$- differentially private for all gradients $\bm{g}\in \calG$. 
\end{proof}
{
\begin{proof}[\bf Proof of Corollary \ref{corollary: constant clipping with d.p.} (Bounded DP with Bounded Sensitivity)]
By Theorem \ref{thm: sto sign dp}, we conclude that the compressor $\calM_{B,\beta}$ is $\max_{\bm{g}\in \calG}\sum_{i=1}^d \log\pth{1 +  \frac{ 1}{\beta + \dist\pth{\bm{g}_i, \calC_B}}}$- differentially private for all gradients $\bm{g}\in \calG$. It turns out that this bound can be relaxed, and we start the derivation from Eq. \eqref{eq: collection}:
\begin{align*}
    \eqref{eq: collection}&\le 1+ \frac{\abth{\bm{g}_i-\bm{g}_i^{\prime}}}{\beta}.
\end{align*}

Now consider the coordinate collection of the gradient pair, by Theorem \ref{thm: composition}, it remains to bound
\begin{align*}
    \sum_{i=1}^d\log\pth{1+ \frac{\abth{\bm{g}_i-\bm{g}_i^{\prime}}}{\beta}}&\le d\log\qth{\frac{1}{d}\sum_{i=1}^d\pth{1+ \frac{\abth{\bm{g}_i-\bm{g}_i^{\prime}}}{\beta}}}~~~[\text{Jensen's inequality}]\\
    &\le d\log\pth{1+ \frac{1}{d\beta}}\\
    &\le \frac{1}{\beta}~~~[\text{follows from} \log(1+x)<x\text{ when }x>0.]
\end{align*}
\end{proof}
}

\subsubsection{Convergence Results}
For ease of presentation, we define $c_0(n,p) = \frac{2(B+\beta)}{p}\sqrt{2\log\frac{6}{3-5c}}+2\sigma\sqrt{\frac{2\log\frac{6}{c}}{{n}}}$ throughout this section.
\begin{proposition}[Bounded Random Variable Variance Bound]
\label{prop: clipped variance upper bound}
Given a random variable $X$ and a clipping threshold $B>0$, if $\mu=\expect{X}\in[-B,B]$, then $\var\pth{\clip\pth{X,B}}\le\var\pth{X}=\sigma^2$.
\end{proposition}
\begin{proof}[\bf Proof of Proposition \ref{prop: clipped variance upper bound}]

\begin{align}
    \var\pth{\clip\pth{X,B}}:=&\expect{\pth{\clip(X,B)-\expect{\clip(X,B)}}^2}\notag\\
    =&\expect{\pth{\clip(X,B)-\expect{X}}^2}-\pth{\expect{\clip(X,B)-X}}^2\notag\\
    \le&\expect{\pth{\clip(X,B)-\expect{X}}^2}.\label{ineq: variance gap 1}
\end{align}
For ease of exposition, we assume $X$ admits a probability density function $f(x)$. General distributions of $X$ can be shown analogously. 
It follows that
\begin{align}
    &\expect{\pth{\clip(X,B)-\expect{X}}^2}\notag\\
    &=\int_B^\infty (B-\mu)^2f(x)\mathd x+\int_{-B}^B(x-\mu)^2f(x)\mathd x+\int_{-\infty}^{-B}(-B-\mu)^2f(x)\mathd x\notag\\
    &\le \int_B^\infty (x-\mu)^2f(x)\mathd x+\int_{-B}^B(x-\mu)^2f(x)\mathd x+\int_{-\infty}^{-B}(x-\mu)^2f(x)\mathd x\notag\\
    &=\var\pth{X}=\sigma^2.\label{ineq: variance gap 2}
\end{align}

Combining \eqref{ineq: variance gap 1} and \eqref{ineq: variance gap 2}, we conclude  $\var\pth{\clip\pth{X,B}}\le\var\pth{X}=\sigma^2$. 
\end{proof}
\subsubsection{Sub-Gaussian { and Heavy-tailed} Distributions}


%
\begin{proof}[\bf Proof of Theorem \ref{thm: sub-gaussian} (Light and Heavy-tailed Sign Error)]
Recall that 
\begin{align*}
\hat{\bm{g}}_{mi}(t)=\begin{cases}
\qth{\calM_{B,\beta}}_i\pth{\frac{1}{n}\sum_{j=1}^n\bm{g}_{mi}^j(t)}~~~ &\text{if ~}m\in\calN(t);\\
\ast~~ &\text{if ~}m\in\calB(t),
\end{cases}
\end{align*}
where $\ast$ is an arbitrary value in \{-1,1\}. 
For any client $m\in [M]$ and any coordinate $i\in [d]$, let
\begin{align*}
X_{mi}&=\indc{m\in\calS(t)}\indc{\hat{\bm{g}}_{mi}\neq\sign\pth{\frac{1}{M}\sum_{m=1}^M\bm{g}_{mi}}}, \\
\text{and} \qquad \tilde{X}_{mi}&=\indc{m\in\calS(t)}\indc{\qth{\calM_\beta}_i\pth{\frac{1}{n}\sum_{j=1}^n\bm{g}_{mi}^j(t)}\neq\sign\pth{\frac{1}{M}\sum_{m=1}^M\bm{g}_{mi}}}.
\end{align*}
Notably, if $m\in \calB(t)$, then it is possible that $X_{mi} \not= \tilde{X}_{mi}$; otherwise,  $X_{mi} = \tilde{X}_{mi}$. 

Without loss of generality, we assume the true aggregation is negative, i.e., $\sign\pth{{ \nabla F_i}\pth{w(t)}}=-1.$
The case when $\sign\pth{{ \nabla F_i}\pth{w(t)}}=1$ can be shown analogously. 

For ease of exposition, we drop a condition of $w(t)$ in the conditional probability expressions unless otherwise noted.  It holds that
\begin{align}
\prob{\sign\pth{\frac{1}{M}\sum_{m=1}^M\hat{\bm{g}}_{mi}}\neq -1}&\le\prob{\sum_{m=1}^M X_{mi}\ge\frac{|\calS(t)|}{2}}\notag\\
    &=\prob{\sum_{m\in\calN(t)}\tilde{X}_{mi}+\sum_{m\in\calB(t)}X_{mi}\ge\frac{|\calS(t)|}{2}}\notag\\
    & = \prob{\sum_{m\in \calN(t)} \tilde{X}_{mi}\ge\frac{|\calS(t)|}{2}-\sum_{m\in\calB(t)}X_{mi}}\notag\\
    &\le\prob{\sum_{m=1}^M \tilde{X}_{mi}\ge\frac{|\calS(t)|}{2}-\sum_{m\in\calB(t)}X_{mi}}. 
    \label{main ineq subgaussian}
\end{align}
Next, we bound $\sum_{m=1}^M \tilde{X}_{mi}$ and $\sum_{m\in\calB(t)}X_{mi}$ separately.

\underline{When the system adversary is static}, i.e., the system adversary does not know $\calS(t)$,  it corrupts clients independently of $\calS(t)$. Hence, 
\begin{align}
\label{eq: each round: static adversary}
\sum_{m\in\calB(t)}X_{mi} \le \sum_{m\in\calB(t)} \indc{m\in\calS(t)}. 
\end{align}
We know that if $\tau(t)\le  \frac{2}{p^2}\log \frac{6}{c}$, then $\sum_{m\in\calB(t)} \indc{m\in\calS(t)}\le \frac{2}{p^2}\log \frac{6}{c}$. Otherwise, with probability at least $1-\frac{c}{6}$, it is true that $\sum_{m\in\calB(t)} \indc{m\in\calS(t)}\le \frac{3}{2}p\tau(t)$. 

On the other hand, \underline{when the system adversary is adaptive}, it chooses $\calB(t)$ based on $\calS(t)$. 
In particular, if $\abth{\calS(t)}\le \tau(t)$, then the adversary chooses $\calB(t) = \calS(t)$. 
Otherwise, i.e., $\abth{\calS(t)}> \tau(t)$, the adversary chooses an arbitrary subset of $\calS(t)$. In both cases, it holds that 
\begin{align}
\label{eq: each round: adaptive adversary}
\sum_{m\in\calB(t)}X_{mi} \le \sum_{m\in\calB(t)} \indc{m\in\calS(t)}\le \min \{\tau(t), |\calS(t)|\} \le \tau(t).
\end{align}

For ease of exposition, we first focus on adaptive adversary and will visit the static adversary towards the end of this proof. 
Observe that $|\calS(t)|=\sum_{m=1}^M\indc{m\in\calS(t)}$. Let $\tilde{Y}_{mi}=\tilde{X}_{mi}-\frac{\indc{m\in\calS(t)}}{2}.$ 
Conditioning on the mini-batch stochastic gradients $\bm{g}_{mi}^1, \cdots, \bm{g}_{mi}^n$, we have 
\begin{align*}
    \expect{\tilde{Y}_{mi}\, \mid\, \bm{g}_{mi}^1, \cdots, \bm{g}_{mi}^n}=\expect{\tilde{X}_{mi}\, \mid\, \bm{g}_{mi}^1, \cdots, \bm{g}_{mi}^n}-\frac{p}{2}=\frac{p}{2B+2\beta}\clip\pth{\frac{1}{n}\sum_{j=1}^n\bm{g}_{mi}^j,B}. 
\end{align*}
Taking expectation over $\bm{g}_{mi}^1, \cdots, \bm{g}_{mi}^n$, we get 
\begin{align}
\label{ineq: double expectation subgaussian}
    \expect{\expect{\tilde{Y}_{mi}\, \mid\, \bm{g}_{mi}^1, \cdots, \bm{g}_{mi}^n}}&=\expect{\expect{\tilde{Y}_{mi}\, \mid\, \bm{g}_{mi}^1, \cdots, \bm{g}_{mi}^n}-p\frac{\frac{1}{n}\sum_{j=1}^n\bm{g}_{mi}^j}{2B+2\beta}}+\frac{p\bm{g}_{mi}}{2B+2\beta}\notag\\
    &=\expect{\expect{\tilde{Y}_{mi}\, \mid\, \bm{g}_{mi}^1, \cdots, \bm{g}_{mi}^n}-p\frac{\frac{1}{n}\sum_{j=1}^n\bm{g}_{mi}^j}{2B+2\beta}}+\frac{p\bm{g}_{mi}}{2B+2\beta}.
\end{align}
It turns out that $\expect{\expect{\tilde{Y}_{mi}\, \mid\, \bm{g}_{mi}^1, \cdots, \bm{g}_{mi}^n}-p\frac{\frac{1}{n}\sum_{j=1}^n\bm{g}_{mi}^j}{2B+2\beta}}$ is small: 
\begin{flalign*}
\begin{aligned}
     &\frac{1}{p}\expect{\expect{\tilde{Y}_{mi}\, \mid\, \bm{g}_{mi}^1, \cdots, \bm{g}_{mi}^n}-p\frac{\frac{1}{n}\sum_{j=1}^n\bm{g}_{mi}^j}{2B+2\beta}}\\
    =&\underbrace{\frac{B\prob{\frac{1}{n}\sum_{j=1}^n\bm{g}_{mi}^j\ge B}-B\prob{\frac{1}{n}\sum_{j=1}^n\bm{g}_{mi}^j\le -B}}{2B+2\beta}}_{(\mathrm{A})}+\underbrace{\frac{\expect{-\frac{1}{n}\sum_{j=1}^n\bm{g}_{mi}^j\indc{\abth{\frac{1}{n}\sum_{j=1}^n\bm{g}_{mi}^j}\ge B}}}{2B+2\beta}}_{\pth{\mathrm{B}}}.
\end{aligned}
\end{flalign*}

{ We bound $\pth{\mathrm{A}}$ and $\pth{\mathrm{B}}$ for sub-Gaussian and heavy-tailed noise separately.

First, for sub-Gaussian distributions with Assumption \ref{ass: sub gaussianity1}, we have
}

\begin{align*}
&
\begin{aligned}
\pth{\mathrm{A}}\le&\frac{B}{2B+2\beta}\prob{\frac{1}{n}\sum_{j=1}^n\bm{g}_{mi}^j-\expect{\frac{1}{n}\sum_{j=1}^n\bm{g}_{mi}^j}\ge B-\expect{\frac{1}{n}\sum_{j=1}^n\bm{g}_{mi}^j}}\\
\le& \frac{B}{2B+2\beta}\exp\pth{-\frac{n\pth{B-\bm{g}_{mi}}^2}{2\sigma_{mi}^2}}\\
\le&\frac{B}{2B+2\beta}\exp\pth{-\frac{n\epsilon_0^2 B_0^2}{2\sigma^2_{mi}}}\\
\le & \frac{1}{2}\exp\pth{-\frac{n}{2}} ~~~[\text{since }\epsilon_0> \frac{\sigma}{B_0}], 
\end{aligned}
\end{align*}
and
\begin{align*}
&
\begin{aligned}
\pth{\mathrm{B}}=&\frac{\expect{-\frac{1}{n}\sum_{j=1}^n\bm{g}_{mi}^j\indc{\abth{\frac{1}{n}\sum_{j=1}^n\bm{g}_{mi}^j}\ge B}}}{2B+2\beta}\\
=&\frac{\int_{-\infty}^{-B}\prob{\frac{1}{n}\sum_{j=1}^n\bm{g}_{mi}^j<t}\mathrm{d}t-\int_{B}^{+\infty}\prob{\frac{1}{n}\sum_{j=1}^n\bm{g}_{mi}^j>t}\mathrm{d}t}{2B+2\beta}\\
\le&\frac{\int_{-\infty}^{-B}\prob{\frac{1}{n}\sum_{j=1}^n\bm{g}_{mi}^j-\expect{\frac{1}{n}\sum_{j=1}^n\bm{g}_{mi}^j}<t-\expect{\frac{1}{n}\sum_{j=1}^n\bm{g}_{mi}^j}}\mathrm{d}t}{2B+2\beta}\\
\le&\frac{\int_{-\infty}^{-B}\exp\pth{-\frac{\pth{t-\bm{g}_{mi}}^2}{2\sigma_{mi}^2/n}}\mathrm{d}t}{2B+2\beta}~~~[\text{Mill's ratio  \cite{gordon1941values}}]\\
\le & \frac{\sigma_{mi}^2/n}{\pth{2B+2\beta} \pth{B+\bm{g}_{mi}}}  \int_{-\infty}^{-B}  \qth{ -\frac{2\pth{t-\bm{g}_{mi}}}{2\sigma_{mi}^2/n}}   \exp\pth{-\frac{\pth{t-\bm{g}_{mi}}^2}{2\sigma_{mi}^2/n}}\mathrm{d}t\\
\le&\frac{\sigma_{mi}^2}{n\epsilon_0B_0(2B+2\beta)}\exp\pth{-\frac{n\epsilon_0^2 B_0^2}{2\sigma^2_{mi}}}\\
\le& \frac{1}{2n} \exp\pth{-\frac{n}{2}}, 
\end{aligned}
\end{align*}
where the last inequality follows from the choice of $\epsilon_0> \frac{\sigma}{B_0}$. 
Combining the bounds of $\pth{\mathrm{A}}$ and $\pth{\mathrm{B}}$, we get 
$\expect{\expect{\tilde{Y}_{mi}\, \mid\, \bm{g}_{mi}^1, \cdots, \bm{g}_{mi}^n}-p\frac{\frac{1}{n}\sum_{j=1}^n\bm{g}_{mi}^j}{2B+2\beta}} \le p\exp\pth{-\frac{n}{2}}$. 
Hence, 
\begin{align}
\label{eq: expectation bound on Y}
\expect{\tilde{Y}_{mi}} \le p\exp\pth{-\frac{n}{2}} + \frac{p\bm{g}_{mi}}{2B+2\beta}.
\end{align}
{
Second, for heavy-tailed distributions with Assumption \ref{ass: heavy tailed}, we have 
\begin{align*}
&
\begin{aligned}
\pth{\mathrm{A}}\le&\frac{B}{2B+2\beta}\prob{\frac{1}{n}\sum_{j=1}^n\bm{g}_{mi}^j-\expect{\frac{1}{n}\sum_{j=1}^n\bm{g}_{mi}^j}\ge B-\expect{\frac{1}{n}\sum_{j=1}^n\bm{g}_{mi}^j}}\\
\le&\frac{B}{2B+2\beta}\frac{\expect{\abth{\sum_{j=1}^n\bm{g}_{mi}^j-\expect{\sum_{j=1}^n\bm{g}_{mi}^j}}^{p^\prime}}}{n^{p^\prime}\abth{B-\bm{g}_{mi}}^{p^\prime}}~~~[\text{Markov's inequality}]\\
\le &\underbrace{ \frac{B\sum_{j=1}^n\expect{\abth{\bm{g}_{mi}^j-\expect{\bm{g}_{mi}^j}}^{p^\prime}}+B\pth{\sum_{j=1}^n\expect{\abth{\bm{g}_{mi}^j-\expect{\bm{g}_{mi}^j}}^2}}^{\frac{{p^\prime}}{2}}}{(2B+2\beta)n^{p^\prime}\abth{B-\bm{g}_{mi}}^{p^\prime}}}_{\text{Rosenthal-type inequality \cite{merlevede2013rosenthal}}}\\
\le & \frac{1}{2}\frac{nM_{p^\prime}+n^{\frac{{p^\prime}}{2}}M_{p^\prime}}{n^{p^\prime}\abth{B-\bm{g}_{mi}}^{p^\prime}}~~~[M_2^{\frac{1}{2}}\le M_{p^\prime}^{\frac{1}{{p^\prime}}}\text{ for }{p^\prime}\ge4]\\
\le & \frac{M_{p^\prime}}{n^{\frac{{p^\prime}}{2}}\epsilon_0^{p^\prime}B_0^{p^\prime}}\le \frac{1}{n^{\frac{{p^\prime}}{2}}}
\end{aligned}\\
\end{align*}
and
\begin{flalign*}
\begin{aligned}
\pth{\mathrm{B}}=&\frac{\expect{-\frac{1}{n}\sum_{j=1}^n\bm{g}_{mi}^j\indc{\abth{\frac{1}{n}\sum_{j=1}^n\bm{g}_{mi}^j}\ge B}}}{2B+2\beta}\\
=&\frac{\int_{-\infty}^{-B}\prob{\frac{1}{n}\sum_{j=1}^n\bm{g}_{mi}^j<t}\mathrm{d}t-\int_{B}^{+\infty}\prob{\frac{1}{n}\sum_{j=1}^n\bm{g}_{mi}^j>t}\mathrm{d}t}{2B+2\beta}\\
\le&\frac{\int_{-\infty}^{-B}\prob{\frac{1}{n}\sum_{j=1}^n\bm{g}_{mi}^j-\expect{\frac{1}{n}\sum_{j=1}^n\bm{g}_{mi}^j}<t-\expect{\frac{1}{n}\sum_{j=1}^n\bm{g}_{mi}^j}}\mathrm{d}t}{2B+2\beta}\\
\le&\frac{1}{2B+2\beta}\int_{-\infty}^{-B}\frac{2M_{p^\prime}}{n^{\frac{{p^\prime}}{2}}\abth{t-\bm{g}_{mi}}^{p^\prime}}\mathrm{d}t~~~[\text{similar argument as in $(\mathrm{A})$}]\\
\le&\frac{1}{2B+2\beta}\frac{1}{\epsilon_0^{{p^\prime}-1}B_0^{{p^\prime}-1}({p^\prime}-1)n^{\frac{{p^\prime}}{2}}} \le \frac{1}{({p^\prime}-1)n^{\frac{{p^\prime}}{2}}}\le  \frac{1}{n^{\frac{{p^\prime}}{2}}}, 
\end{aligned}
\end{flalign*}
where the last inequality follows from the choice of $\epsilon_0> \frac{M_{p^\prime}^{\frac{1}{{p^\prime}}}}{B_0}$. 
Combining the bounds of $\pth{\mathrm{A}}$ and $\pth{\mathrm{B}}$, we get 
$\expect{\expect{\tilde{Y}_{mi}\, \mid\, \bm{g}_{mi}^1, \cdots, \bm{g}_{mi}^n}-p\frac{\frac{1}{n}\sum_{j=1}^n\bm{g}_{mi}^j}{2B+2\beta}} \le \frac{2p}{n^{\frac{{p^\prime}}{2}}}$. 
Hence, 
\begin{align}
\label{eq: expectation bound on Y heavy tailed}
\expect{\tilde{Y}_{mi}} \le \frac{2p}{n^{\frac{{p^\prime}}{2}}} + \frac{p\bm{g}_{mi}}{2B+2\beta}.
\end{align}
}

Let us consider two mutually complement events $\calE_1$ and $\calE_2$:
\begin{small}
\begin{align*}
     &\calE_1: = \sth{\frac{1}{2(B+\beta)}\sum_{m=1}^M\clip\pth{\frac{1}{n}\sum_{j=1}^n\bm{g}_{mi}^j,B} -\expect{\frac{1}{2(B+\beta)}  \sum_{m=1}^M \clip\pth{\frac{1}{n}\sum_{j=1}^n\bm{g}_{mi}^j,B}}\le \frac{{ c_0(n,p)}}{4(B+\beta)} \sqrt{M}}, \\
    &\calE_2: = \sth{\frac{1}{2(B+\beta)}\sum_{m=1}^M\clip\pth{\frac{1}{n}\sum_{j=1}^n\bm{g}_{mi}^j,B} -\expect{\frac{1}{2(B+\beta)}  \sum_{m=1}^M \clip\pth{\frac{1}{n}\sum_{j=1}^n\bm{g}_{mi}^j,B}}> \frac{{ c_0(n,p)}}{4(B+\beta)} \sqrt{M}}.
\end{align*}
\end{small}
We have
\begin{align}
   \prob{\sum_{m=1}^M \tilde{X}_{mi}\ge\frac{|\calS(t)|}{2}-\tau(t)}
    &\le\prob{\sum_{m=1}^M\tilde{Y}_{mi}\ge-\tau(t)~\mid~\calE_1}
    +\prob{\calE_2}.\label{ineq: Berstein byzantine clip subgaussian}
\end{align}

By Proposition \ref{prop: clipped variance upper bound}, we know that 
\begin{align*}
\var\pth{\clip\pth{\frac{1}{n}\sum_{j=1}^n\bm{g}_{mi}^j,B}}  
\le \var\pth{\frac{1}{n}\sum_{j=1}^n\bm{g}_{mi}^j} 
\le \frac{1}{n} \var\pth{\bm{g}_{mi}^1} = \frac{1}{n} \sigma^2_{mi} \le \frac{1}{n}\sigma^2. 
\end{align*}
In addition, $\clip\pth{\frac{1}{n}\sum_{j=1}^n\bm{g}_{mi}^j,B}$ is { bounded and thus} sub-Gaussian. { Hence}, we have 
\begin{align*}
\prob{\calE_2} 
\le& \exp\pth{-\frac{\frac{{ c_0^2(n,p)}M}{4}}{\frac{2M\sigma^2}{n}}}. 
\end{align*}
Since ${ c_0(n,p)}\ge \sqrt{\frac{8\sigma^2}{n} \log \frac{6}{c}}$, we have $\prob{\calE_2} \le \frac{c}{6}$. 

For the first term in the right-hand side of Eq.\,\eqref{ineq: Berstein byzantine clip subgaussian}, we have 
\begin{align*}
&\prob{\sum_{m=1}^M\tilde{Y}_{mi}\ge-\tau(t)~\mid~\calE_1}\\
=&\prob{\sum_{m=1}^M\tilde{Y}_{mi}-\expect{\sum_{m=1}^M\tilde{Y}_{mi} \, \mid\, \bm{g}_{mi}^1, \cdots, \bm{g}_{mi}^n}\ge\underbrace{-\tau(t)-\expect{\sum_{m=1}^M\tilde{Y}_{mi}\, \mid\, \bm{g}_{mi}^1, \cdots, \bm{g}_{mi}^n}}_{\pth{\mathrm{C}}}~\mid~\calE_1}
\end{align*}
Recall that $\expect{\tilde{Y}_{mi}\, \mid\, \bm{g}_{mi}^1, \cdots, \bm{g}_{mi}^n}=\frac{p}{2B+2\beta}\clip\pth{\frac{1}{n}\sum_{j=1}^n\bm{g}_{mi}^j,B}. $
We have
\begin{align*}
    \pth{\mathrm{C}} \mid \calE_1 &= -\tau(t) - \frac{p}{2B+2\beta}\sum_{m=1}^M\clip\pth{\frac{1}{n}\sum_{j=1}^n\bm{g}_{mi}^j,B} \mid \calE_1 \\
    & \ge  -\tau(t) - \expect{\frac{p}{2B+2\beta}\sum_{m=1}^M\clip\pth{\frac{1}{n}\sum_{j=1}^n\bm{g}_{mi}^j,B}} -\frac{p{ c_0(n,p)}}{4(B+\beta)}\sqrt{M}\\
    & = -\tau(t) - \sum_{m=1}^M\expect{\tilde{Y}_{mi}} -\frac{p{ c_0(n,p)}}{4(B+\beta)}\sqrt{M}\\
    &{ \begin{cases}
        &\ge -\tau(t) - Mp\exp\pth{-\frac{n}{2}} - \frac{pM}{2(B+\beta)}{ \nabla F_i}(w(t))-\frac{p{ c_0(n,p)}}{4(B+\beta)}\sqrt{M}~~~[\text{Sub-Gaussian Noise}]\\
        &\ge -\tau(t) - \frac{2Mp}{n^{\frac{{p^\prime}}{2}}}  - \frac{pM}{2(B+\beta)}{ \nabla F_i}(w(t))-\frac{p{ c_0(n,p)}}{4(B+\beta)}\sqrt{M}~~~[\text{Heavy-tailed Noise}]
    \end{cases} }
\end{align*}
Recall that ${ \nabla F_i}(w(t))<0$. 

\noindent When $\frac{pM}{2(B+\beta)}\abth{{ \nabla F_i}(w(t))} \ge \tau(t)+ Mp\exp\pth{-\frac{n}{2}}+ \frac{p{ c_0(n,p)}}{2(B+\beta)}\sqrt{M}$ { (sub-Gaussian noise) or when  $\frac{pM}{2(B+\beta)}\abth{{ \nabla F_i}(w(t))} \ge \tau(t)+ \frac{2Mp}{n^{\frac{{p^\prime}}{2}}}+ \frac{p{ c_0(n,p)}}{2(B+\beta)}\sqrt{M}$ (heavy-tailed noise)}, we get
\begin{align*}
\prob{\sum_{m=1}^M\tilde{Y}_{mi}\ge-\tau(t)~\mid~\calE_1} \le& \prob{\sum_{m=1}^M\tilde{Y}_{mi}-\expect{\sum_{m=1}^M\tilde{Y}_{mi} \, \mid\, \bm{g}_{mi}^1, \cdots, \bm{g}_{mi}^n}\ge \frac{p{ c_0(n,p)}}{4(B+\beta)}\sqrt{M} ~\mid~\calE_1}\\
\le& \exp\pth{-\frac{p^2{ c_0^2(n,p)}}{8(B+\beta)^2}} \\
\le& \frac{3-5c}{6}, 
\end{align*}
where the last inequality holds because  ${ c_0(n,p)}\ge \sqrt{\frac{8\pth{B+\beta}^2}{p^2} \log \frac{6}{3-5c}}$. 


It remains to show the case for static adversary. When $\tau(t) \le  \frac{2}{p^2}\log \frac{6}{c}$, we bound Eq.\,\eqref{main ineq subgaussian} as 
\begin{align*}
\prob{\sum_{m=1}^M \tilde{X}_{mi}\ge\frac{|\calS(t)|}{2}-\sum_{m\in\calB(t)}X_{mi}}
\le & \prob{\sum_{m=1}^M \tilde{X}_{mi}\ge\frac{|\calS(t)|}{2}-\tau(t)}.
\end{align*}
When $\tau(t)> \frac{2}{p^2}\log \frac{6}{c}$, we bound Eq.\,\eqref{main ineq subgaussian} as 
\begin{align*}
\prob{\sum_{m=1}^M \tilde{X}_{mi}\ge\frac{|\calS(t)|}{2}-\sum_{m\in\calB(t)}X_{mi}}
\le & \prob{\sum_{m=1}^M \tilde{X}_{mi}\ge\frac{|\calS(t)|}{2}-\frac{3p}{2}\tau(t)} + \frac{c}{6}.
\end{align*}
The remaining proof follows the above argument for adaptive adversary. 
\end{proof}


\begin{proof}[\bf Proof of Theorem \ref{thm: convergence subgaussian} (Sub-Gaussian { and Heavy-tailed} Convergence Rate)]

By Assumption \ref{ass: 2 smmothness}, we have
\begin{align*}
    F\pth{w(t+1)}-F\pth{w(t)}&\le\langle\nabla F(w(t)),w(t+1)-w(t)\rangle+\frac{L}{2}\|w(t+1)-w(t)\|^2\\
    &=-\eta_t\|\nabla F(w(t))\|_1+2\eta\sum_{i=1}^d\abth{\nabla F(w(t))_i}\indc{\tilde{\bm{g}}_i\neq \sign{ \pth{\nabla F(w(t))_i}}}+\frac{Ld}{2}\eta_t^2,
\end{align*}
where $\nabla F(w(t))_i$ is the $i$-th coordinate of $\nabla F(w(t))$. 
Then, by conditioning on parameter $w(t)$, we get
\begin{align*}
    &\expect{F\pth{w(t+1)}-F\pth{w(t)}\big |w(t)}\\
    &\le \expect{-\eta_t\|\nabla F(w(t))\|_1+2\eta_t\sum_{i=1}^d\abth{\nabla F(w(t))_i}\indc{\tilde{\bm{g}}_i\neq \sign\pth{\nabla F(w(t))_i}}+\frac{Ld}{2}\eta_t^2}\\
    &=-\eta_t\|\nabla F(w(t))\|_1+\frac{Ld}{2}\eta_t^2+2\eta_t\sum_{i=1}^d\abth{\nabla F(w(t))_i}\prob{\tilde{\bm{g}}_i\neq \sign\pth{\nabla F(w(t))_i}}.
\end{align*}

\noindent { Recall that $\Xi_1(n)= 2(B+\beta)\exp\pth{-\frac{n}{2}}$, and $\Xi_2(n)=\frac{4(B+\beta)}{n^{\frac{p^\prime}{2}}}$. Define 
\begin{align*}
    \begin{cases}
        &\mathrm{A}_1=\sth{\abth{{ \nabla F}(w(t))_i} \ge \frac{2(B+\beta)}{pM}\tau(t) + 2(B+\beta)\exp\pth{-\frac{n}{2}} + \frac{2(B+\beta)}{p}\sqrt{\frac{2}{M}\log\frac{6}{3-5c}}+2\sigma\sqrt{\frac{2\log\frac{6}{c}}{{Mn}}}};\\
        &\mathrm{A}_2= \sth{\abth{{ \nabla F}(w(t))_i} \ge \frac{2(B+\beta)}{pM}\tau(t) + \frac{4(B+\beta)}{n^{\frac{p^\prime}{2}}} + \frac{2(B+\beta)}{p}\sqrt{\frac{2}{M}\log\frac{6}{3-5c}}+2\sigma\sqrt{\frac{2\log\frac{6}{c}}{{Mn}}}}.
    \end{cases}
\end{align*}
In the following proof, we denote $\mathrm{A}=\mathrm{A}_1$, $\Xi(n)=\Xi_1(n)$ for sub-Gaussian noise and $\mathrm{A}=\mathrm{A}_2$, $\Xi(n)=\Xi_2(n)$ for heavy-tailed noise.
}

We now have two cases:

\noindent\underline{First}, when the system adversary is adaptive or the system adversary is static but with $\tau(t)\le \frac{2}{p^2}\log\frac{6}{c}$, then 
{
\begin{align}
    &\expect{F\pth{w(t+1)}-F\pth{w(t)}\big |w(t)}\notag\\
    &\le-\eta_t\|\nabla F(w(t))\|_1+\frac{Ld}{2}\eta_t^2\notag\\
    &~~~~+2\eta_t\sum_{i=1}^d\abth{\nabla F(w(t))_i}\frac{1-c}{2}\indc{\mathrm{A}}\notag\\
    &~~~~+2\eta_t\sum_{i=1}^d\qth{\frac{2(B+\beta)\tau(t)}{pM}+\frac{2(B+\beta)}{p}\sqrt{\frac{2}{M}\log\frac{6}{3-5c}}+2\sigma\sqrt{\frac{2\log\frac{6}{c}}{{Mn}}}+\Xi(n)}\indc{\mathrm{A}^\complement}\notag\\
    &\le-\eta_t c\|\nabla F(w(t))\|_1+\frac{Ld}{2}\eta^2+\frac{4\eta_t d(B+\beta)}{p}\sqrt{\frac{2}{M}\log\frac{6}{3-5c}}+4\sigma\eta_t d\sqrt{\frac{2\log\frac{6}{c}}{{Mn}}}\notag\\
    &\quad+4\eta_t d\frac{(B+\beta)\tau(t)}{pM}+2\eta_t d \Xi(n).\notag
\end{align}
}
Therefore, by Assumption \ref{ass: 1 lower bound}, we have
{ 
\begin{align*}
    &F^*-F(w(0))\\
    &\le\expect{F\pth{w(T)}-F\pth{w(0)}}\\
    &\le- c\sum_{t=0}^{T-1}\eta_t\expect{\|\nabla F(w(t))\|_1}+\frac{Ld\sum_{t=0}^{T-1}\eta^2}{2}
    +\frac{4 d(B+\beta)\sum_{t=0}^{T-1}\eta_t}{p}\sqrt{\frac{2}{M}\log\frac{6}{3-5c}}+4\sigma d\sum_{t=0}^{T-1}\eta_t \sqrt{\frac{2\log\frac{6}{c}}{{Mn}}}\\
    &\quad+2 d \Xi(n)\sum_{t=0}^{T-1}+4d\frac{(B+\beta)\sum_{t=0}^{T-1}\eta_t\tau(t)}{pM}.
\end{align*}
}

Rearrange the inequality, we get
\begin{align*}
    \sum_{t=0}^{T-1}\eta_t\expect{\|\nabla F(w(t))\|_1}\le \frac{1}{c}\left[
    F(w(0)) - F^* + \frac{Ld\sum_{t=0}^{T-1}\eta^2}{2} + \frac{4 d(B+\beta)\sum_{t=0}^{T-1}\eta_t}{p}\sqrt{\frac{2}{M}\log\frac{6}{3-5c}}\right.&\\
    \left.+4\sigma d\sum_{t=0}^{T-1}\eta_t \sqrt{\frac{2\log\frac{6}{c}}{{Mn}}} + 2 d \Xi(n)\sum_{t=0}^{T-1}\eta_t+4d\frac{(B+\beta)\sum_{t=0}^{T-1}\eta_t\tau(t)}{pM}\right]&
\end{align*}


It follows that
\begin{align*}
    \expect{\|\nabla F(w(R))\|_1}\le \frac{1}{c}\left[
    \frac{F(w(0)) - F^*}{\sum_{t=0}^{T-1}\eta_t} + \frac{Ld\sum_{t=0}^{T-1}\eta_t^2}{2\sum_{t=0}^{T-1}\eta_t} + \frac{4 d(B+\beta)}{p}\sqrt{\frac{2}{M}\log\frac{6}{3-5c}}\right.&\\
    \left.+4\sigma d\sqrt{\frac{2\log\frac{6}{c}}{{Mn}}} + 2 d \Xi(n)+4d\frac{(B+\beta)\sum_{t=0}^{T-1}\eta_t\tau(t)}{pM\sum_{t=0}^{T-1}\eta_t}\right]&
\end{align*}
\noindent\underline{Second}, when the system adversary is static with $\tau(t) >\frac{2}{p^2}\log\frac{6}{c}$, { follow a similar proof as above, we get

\begin{align*}
    \expect{\|\nabla F(w(R))\|_1}\le \frac{1}{c}\left[
    \frac{F(w(0)) - F^*}{\sum_{t=0}^{T-1}\eta_t} + \frac{Ld\sum_{t=0}^{T-1}\eta_t^2}{2\sum_{t=0}^{T-1}\eta_t} + \frac{4 d(B+\beta)}{p}\sqrt{\frac{2}{M}\log\frac{6}{3-5c}}\right.&\\
    \left.+4\sigma d\sqrt{\frac{2\log\frac{6}{c}}{{Mn}}} + 2 d \Xi(n)+6d\frac{(B+\beta)\sum_{t=0}^{T-1}\eta_t\tau(t)}{M\sum_{t=0}^{T-1}\eta_t}\right]&
\end{align*}
}
\end{proof}
\subsubsection{Gaussian Distribution}


\begin{proof}[\bf Proof of Corollary \ref{corollary: gaussian} (Gaussian Tail Sign Errors)]
Most of the proofs are the same with Theorem \ref{thm: sub-gaussian}.
We start from Eq. \,\ref{ineq: double expectation subgaussian} and define $c_0(n,p) = \frac{2(B+\beta)}{p}\sqrt{2\log\frac{6}{3-5c}}+2\sigma\sqrt{\frac{2\log\frac{6}{c}}{{n}}}.$
.

It turns out that $\expect{\expect{\tilde{Y}_{mi}\, \mid\, \bm{g}_{mi}^1, \cdots, \bm{g}_{mi}^n}-p\frac{\frac{1}{n}\sum_{j=1}^n\bm{g}_{mi}^j}{2B+2\beta}}$ is small: 
\begin{flalign}
\label{eq: combing loss gaussian}
    & \frac{1}{p}\expect{\expect{\tilde{Y}_{mi}\, \mid\, \bm{g}_{mi}^1, \cdots, \bm{g}_{mi}^n}-p\frac{\frac{1}{n}\sum_{j=1}^n\bm{g}_{mi}^j}{2B+2\beta}}\notag&\\
    &=\underbrace{\frac{\pth{B-\bm{g}_{mi}}\prob{\frac{1}{n}\sum_{j=1}^n\bm{g}_{mi}^j\ge B}}{2B+2\beta}}_{\pth{\mathrm{A}}}-\underbrace{\frac{\pth{B+\bm{g}_{mi}}\prob{\frac{1}{n}\sum_{j=1}^n\bm{g}_{mi}^j\le -B}}{2B+2\beta}}_{\pth{\mathrm{B}}}\notag&\\
  &+\underbrace{\frac{\expect{\pth{-\frac{1}{n}\sum_{j=1}^n\bm{g}_{mi}^j+\bm{g}_{mi}}\indc{\abth{\frac{1}{n}\sum_{j=1}^n\bm{g}_{mi}^j}\ge B}}}{2B+2\beta}}_{\pth{\mathrm{C}}}.&
\end{flalign}

We have,
\begin{flalign*}
 &(2B+2\beta)\pth{\mathrm{A}}&\\
 &\le\pth{B-\bm{g}_{mi}}\cdot\frac{\sigma_{mi}/\sqrt{n}}{B-\bm{g}_{mi}}\cdot\frac{1}{\sqrt{2\pi}}\cdot\exp\pth{-\frac{(B-\bm{g}_{mi})^2}{2\pth{\sigma_{mi}/\sqrt{n}}^2}}=\frac{\sigma_{mi}/\sqrt{n}}{\sqrt{2\pi}}\exp\pth{-\frac{(B-\bm{g}_{mi})^2}{2\pth{\sigma_{mi}/\sqrt{n}}^2}}&;
 \end{flalign*}
 \begin{flalign*}
&(2B+2\beta)\pth{\mathrm{B}}&\\
&\ge\pth{B+\bm{g}_{mi}}\cdot\frac{\frac{B+\bm{g}_{mi}}{\sigma_{mi}/\sqrt{n}}}{\pth{\frac{B+\bm{g}_{mi}}{\sigma_{mi}/\sqrt{n}}}^2+1}\cdot\frac{1}{\sqrt{2\pi}}\cdot\exp\pth{-\frac{(B+\bm{g}_{mi})^2}{2\pth{\sigma_{mi}/\sqrt{n}}^2}}&\\
 &=\left[1-\frac{\pth{\sigma_{mi}/\sqrt{n}}^2}{(B+\bm{g}_{mi})^2+\pth{\sigma_{mi}/\sqrt{n}}^2}\right]\frac{\sigma_{mi}/\sqrt{n}}{\sqrt{2\pi}}\exp\pth{-\frac{(B+\bm{g}_{mi})^2}{2\pth{\sigma_{mi}/\sqrt{n}}^2}};&
\end{flalign*}

\begin{flalign*}
&(2B+2\beta)\pth{\mathrm{C}}&\\
&=-\int_{B}^\infty\frac{x-\bm{g}_{mi}}{\sqrt{2\pi}\sigma_{mi}/\sqrt{n}}\exp\pth{-\frac{(x-\bm{g}_{mi})^2}{2\pth{\sigma_{mi}/\sqrt{n}}^2}}\mathd x-\int_{-\infty}^{-B}\frac{x-\bm{g}_{mi}}{\sqrt{2\pi}\sigma_{mi}/\sqrt{n}}\exp\pth{-\frac{(x-\bm{g}_{mi})^2}{2\pth{\sigma_{mi}/\sqrt{n}}^2}}\mathd x;&\\
&=\frac{\sigma_{mi}/\sqrt{n}}{\sqrt{2\pi}}\left[\exp\pth{-\frac{(B+\bm{g}_{mi})^2}{2\pth{\sigma_{mi}/\sqrt{n}}^2}}-\exp\pth{-\frac{(B-\bm{g}_{mi})^2}{2\pth{\sigma_{mi}/\sqrt{n}}^2}}\right],&
\end{flalign*}

where $(\mathrm{A})$ and $(\mathrm{B})$ follow because of Mill's ratio  \cite{gordon1941values}.

Combining $\pth{\mathrm{A}}$, $\pth{\mathrm{B}}$, and $\pth{\mathrm{C}}$, we get
\begin{align*}
    \eqref{eq: combing loss gaussian}\le&\frac{p\pth{\sigma_{mi}/\sqrt{n}}^3}{\sqrt{2\pi}\pth{2B+2\beta}\qth{\pth{B+\bm{g}_{mi}}^2+\pth{\sigma_{mi}/\sqrt{n}}^2}}\exp\pth{-\frac{(B+\bm{g}_{mi})^2}{2\pth{\sigma_{mi}/\sqrt{n}}^2}}+\frac{p\bm{g}_{mi}}{2B+2\beta}\\
    \le&\frac{p\pth{\sigma_{mi}/\sqrt{n}}^3}{\sqrt{2\pi}\pth{2B+2\beta}\qth{\epsilon_0^2B_0^2+\pth{\sigma_{mi}/\sqrt{n}}^2}}\exp\pth{-\frac{\epsilon_0^2B_0^2}{2\pth{\sigma_{mi}/\sqrt{n}}^2}}+\frac{p\bm{g}_{mi}}{2B+2\beta}\\
    \le&\frac{p}{4\sqrt{2\pi}}\exp\pth{-\frac{n}{2}}+\frac{p\bm{g}_{mi}}{2B+2\beta},
\end{align*}
where the last inequality follows because $\epsilon_0>\frac{\sigma}{B_0}$ and $B:=B_0+\epsilon_0B_0>\epsilon_0B_0.$

For the first term in the right hand side of Eq.\,\eqref{ineq: Berstein byzantine clip subgaussian}, we have 
\begin{align*}
&\prob{\sum_{m=1}^M\tilde{Y}_{mi}\ge-\tau(t)~\mid~\calE_1}\\
=&\prob{\sum_{m=1}^M\tilde{Y}_{mi}-\expect{\sum_{m=1}^M\tilde{Y}_{mi} \, \mid\, \bm{g}_{mi}^1, \cdots, \bm{g}_{mi}^n}\ge\underbrace{-\tau(t)-\expect{\sum_{m=1}^M\tilde{Y}_{mi}\, \mid\, \bm{g}_{mi}^1, \cdots, \bm{g}_{mi}^n}}_{\pth{\mathrm{D}}}~\mid~\calE_1}
\end{align*}
Recall that $\expect{\tilde{Y}_{mi}\, \mid\, \bm{g}_{mi}^1, \cdots, \bm{g}_{mi}^n}=\frac{p}{2B+2\beta}\clip\pth{\frac{1}{n}\sum_{j=1}^n\bm{g}_{mi}^j,B}. $
We have
\begin{align*}
    \pth{\mathrm{D}} \mid \calE_1 &= -\tau(t) - \frac{p}{2B+2\beta}\sum_{m=1}^M\clip\pth{\frac{1}{n}\sum_{j=1}^n\bm{g}_{mi}^j,B} \mid \calE_1 \\
    & \ge  -\tau(t) - \expect{\frac{p}{2B+2\beta}\sum_{m=1}^M\clip\pth{\frac{1}{n}\sum_{j=1}^n\bm{g}_{mi}^j,B}} -\frac{p{ c_0(n,p)}}{4(B+\beta)}\sqrt{M}\\
    & = -\tau(t) - \sum_{m=1}^M\expect{\tilde{Y}_{mi}} -\frac{p{ c_0(n,p)}}{4(B+\beta)}\sqrt{M}\\
    & \ge -\tau(t) - \frac{Mp}{4\sqrt{2\pi}}\exp\pth{-\frac{n}{2}} - \frac{p}{2(B+\beta)}\sum_{m=1}^M\bm{g}_{mi}-\frac{p{ c_0(n,p)}}{4(B+\beta)}\sqrt{M}
\end{align*}
Recall that ${ \nabla F_i}(w(t))<0$. 
When $\frac{Mp}{2(B+\beta)}\abth{{ \nabla F_i}(w(t))} \ge \tau(t)+ \frac{Mp}{4\sqrt{2\pi}}\exp\pth{-\frac{n}{2}}+ \frac{p{ c_0(n,p)}}{2(B+\beta)}\sqrt{M}$, we get
\begin{align*}
\prob{\sum_{m=1}^M\tilde{Y}_{mi}\ge-\tau(t)~\mid~\calE_1} \le& \prob{\sum_{m=1}^M\tilde{Y}_{mi}-\expect{\sum_{m=1}^M\tilde{Y}_{mi} \, \mid\, \bm{g}_{mi}^1, \cdots, \bm{g}_{mi}^n}\ge \frac{p{ c_0(n,p)}}{4(B+\beta)}\sqrt{M} ~\mid~\calE_1}\\
\le& \exp\pth{-\frac{p^2{ c_0^2(n,p)}}{8(B+\beta)^2}} \\
\le& \frac{3-5c}{6}, 
\end{align*}
where the last inequality holds because  ${ c_0(n,p)}\ge \sqrt{\frac{8\pth{B+\beta}^2}{p^2} \log \frac{6}{3-5c}}$. 

The remaining proof follows the arguments in the proof of Theorem \ref{thm: sub-gaussian}. 
\end{proof}



\begin{proof}[\bf Proof of Corollary \ref{corollary: convergence gaussian} (Gaussian Tail Convergence Rate)]
This proof follows from Theorem \ref{thm: convergence subgaussian}. 


We also consider two cases here.

\noindent\underline{First}, when the system adversary is adaptive or the system adversary is static but with $\tau(t)\le \frac{2}{p^2}\log\frac{6}{c}$, plug in $\abth{{ \nabla F_i}(w(t))} \ge \frac{2(B+\beta)}{pM}\tau(t)+ \frac{B+\beta}{2\sqrt{2\pi}}\exp\pth{-\frac{n}{2}}+ \frac{2(B+\beta)}{p}\sqrt{\frac{2}{M}\log\frac{6}{3-5c}}+2\sigma\sqrt{\frac{2\log\frac{6}{c}}{{Mn}}}$, we get
\begin{align*}
    \expect{\|\nabla F(w(R))\|_1}\le \frac{1}{c}\left[
    \frac{F(w(0)) - F^*}{\sum_{t=0}^{T-1}\eta_t} + \frac{Ld\sum_{t=0}^{T-1}\eta_t^2}{2\sum_{t=0}^{T-1}\eta_t} + \frac{4 d(B+\beta)}{p}\sqrt{\frac{2}{M}\log\frac{6}{3-5c}}\right.&\\
    \left.+4\sigma d\sqrt{\frac{2\log\frac{6}{c}}{{Mn}}} + \frac{d}{\sqrt{2\pi}}(B+\beta)\exp\pth{-\frac{n}{2}}+4d\frac{(B+\beta)\sum_{t=0}^{T-1}\eta_t\tau(t)}{pM\sum_{t=0}^{T-1}\eta_t}\right]&
\end{align*}

\noindent\underline{Second}, when the system adversary is static with $\tau(t) >\frac{2}{p^2}\log\frac{6}{c}$, plug in $\abth{{ \nabla F_i}(w(t))} \ge \frac{3(B+\beta)\tau(t)}{M}+ \frac{B+\beta}{2\sqrt{2\pi}}\exp\pth{-\frac{n}{2}}+ \frac{2(B+\beta)}{p}\sqrt{\frac{2}{M}\log\frac{6}{3-5c}}+2\sigma\sqrt{\frac{2\log\frac{6}{c}}{{Mn}}}$, we get
\begin{align*}
    \expect{\|\nabla F(w(R))\|_1}\le \frac{1}{c}\left[
    \frac{F(w(0)) - F^*}{\sum_{t=0}^{T-1}\eta_t} + \frac{Ld\sum_{t=0}^{T-1}\eta_t^2}{2\sum_{t=0}^{T-1}\eta_t} + \frac{4 d(B+\beta)}{p}\sqrt{\frac{2}{M}\log\frac{6}{3-5c}}\right.&\\
    \left.+4\sigma d\sqrt{\frac{2\log\frac{6}{c}}{{Mn}}} + \frac{d}{\sqrt{2\pi}}(B+\beta)\exp\pth{-\frac{n}{2}}+6d\frac{(B+\beta)\sum_{t=0}^{T-1}\eta_t\tau(t)}{M\sum_{t=0}^{T-1}\eta_t}\right]&
\end{align*}
\end{proof}

\subsection{Bounded Stochastic Gradients}



\begin{proof}[\bf Proof of Corollary \ref{corollary: sampling probability with all byzantine}  (Bounded Gradient Sign Errors)]
This proof follows from Theorem \ref{thm: sub-gaussian}. Notably, if we choose $B=\tilde{B}$, $\clip\pth{\frac{1}{n}\sum_{j=1}^n\bm{g}_{mi}^j,B}=\frac{1}{n}\sum_{j=1}^n\bm{g}_{mi}^j$ by Assumption \ref{ass: Bounded stochastic gradient}. Thus, the bias introduced by the tail bound will be gone. 

For the first term in the RHS of Eq.\,\eqref{ineq: Berstein byzantine clip subgaussian}, we have 
\begin{align*}
&\prob{\sum_{m=1}^M\tilde{Y}_{mi}\ge-\tau(t)~\mid~\calE_1}\\
=&\prob{\sum_{m=1}^M\tilde{Y}_{mi}-\expect{\sum_{m=1}^M\tilde{Y}_{mi} \, \mid\, \bm{g}_{mi}^1, \cdots, \bm{g}_{mi}^n}\ge\underbrace{-\tau(t)-\expect{\sum_{m=1}^M\tilde{Y}_{mi}\, \mid\, \bm{g}_{mi}^1, \cdots, \bm{g}_{mi}^n}}_{\pth{\mathrm{A}}}~\mid~\calE_1}
\end{align*}
Recall that $\expect{\tilde{Y}_{mi}\, \mid\, \bm{g}_{mi}^1, \cdots, \bm{g}_{mi}^n}=\frac{p}{2B+2\beta}\frac{1}{n}\sum_{j=1}^n\bm{g}_{mi}^j. $
We have
\begin{align*}
    \pth{\mathrm{A}} \mid \calE_1 &= -\tau(t) - \frac{p}{2B+2\beta}\sum_{m=1}^M\frac{1}{n}\sum_{j=1}^n\bm{g}_{mi}^j \mid \calE_1 \\
    & \ge  -\tau(t) - \expect{\frac{p}{2B+2\beta}\sum_{m=1}^M\frac{1}{n}\sum_{j=1}^n\bm{g}_{mi}^j} -\frac{p{ c_0(n,p)}}{4(B+\beta)}\sqrt{M}\\
    & = -\tau(t) - \sum_{m=1}^M\expect{\tilde{Y}_{mi}} -\frac{p{ c_0(n,p)}}{4(B+\beta)}\sqrt{M}\\
    & \ge -\tau(t) - \frac{p}{2(B+\beta)}\sum_{m=1}^M\bm{g}_{mi}-\frac{p{ c_0(n,p)}}{4(B+\beta)}\sqrt{M}
\end{align*}
Recall that ${ \nabla F_i}(w(t))<0$. 
When $\abth{{ \nabla F_i}(w(t))} \ge \frac{2(B+\beta)\tau(t)}{Mp}+ \frac{{ c_0(n,p)}}{\sqrt{M}}$, we get
\begin{align*}
\prob{\sum_{m=1}^M\tilde{Y}_{mi}\ge-\tau(t)~\mid~\calE_1} \le& \prob{\sum_{m=1}^M\tilde{Y}_{mi}-\expect{\sum_{m=1}^M\tilde{Y}_{mi} \, \mid\, \bm{g}_{mi}^1, \cdots, \bm{g}_{mi}^n}\ge \frac{p{ c_0(n,p)}}{4(B+\beta)}\sqrt{M} ~\mid~\calE_1}\\
\le& \exp\pth{-\frac{p^2{ c_0^2(n,p)}}{8(B+\beta)^2}} \\
\le& \frac{3-5c}{6}, 
\end{align*}
The remaining proof also follows the arguments in the proof of Theorem \ref{thm: sub-gaussian}. 
\end{proof}

\begin{proof}[\bf Proof of Corollary \ref{corollary: convergence all Byzantine} (Bounded Gradient Convergence Rate)]
This proof follows from Theorem \ref{thm: convergence subgaussian}. We also consider two cases here.
\noindent\underline{First}, when the system adversary is adaptive or the system adversary is static but with $\tau(t)\le \frac{2}{p^2}\log\frac{6}{c}$, plug in $\abth{{ \nabla F_i}(w(t))} \ge \frac{2(B+\beta)}{pM}\tau(t)+ \frac{2(B+\beta)}{p}\sqrt{\frac{2}{M}\log\frac{6}{3-5c}}+2\sigma\sqrt{\frac{2\log\frac{6}{c}}{{Mn}}}$, we get
\begin{align*}
    \expect{\|\nabla F(w(R))\|_1}\le \frac{1}{c}\left[
    \frac{F(w(0)) - F^*}{\sum_{t=0}^{T-1}\eta_t} + \frac{Ld\sum_{t=0}^{T-1}\eta_t^2}{2\sum_{t=0}^{T-1}\eta_t} + \frac{4 d(B+\beta)}{p}\sqrt{\frac{2}{M}\log\frac{6}{3-5c}}\right.&\\
    \left.+4\sigma d\sqrt{\frac{2\log\frac{6}{c}}{{Mn}}} +4d\frac{(B+\beta)\sum_{t=0}^{T-1}\eta_t\tau(t)}{pM\sum_{t=0}^{T-1}\eta_t}\right]&
\end{align*}

\noindent\underline{Second}, when the system adversary is static with $\tau(t) >\frac{2}{p^2}\log\frac{6}{c}$, plug in $\abth{{ \nabla F_i}(w(t))} \ge \frac{3(B+\beta)}{M}\tau(t)+ \frac{B+\beta}{2\sqrt{2\pi}}\exp\pth{-\frac{n}{2}}+ \frac{2(B+\beta)}{p}\sqrt{\frac{2}{M}\log\frac{6}{3-5c}}+2\sigma\sqrt{\frac{2\log\frac{6}{c}}{{Mn}}}$, we get
\begin{align*}
    \expect{\|\nabla F(w(R))\|_1}\le \frac{1}{c}\left[
    \frac{F(w(0)) - F^*}{\sum_{t=0}^{T-1}\eta_t} + \frac{Ld\sum_{t=0}^{T-1}\eta_t^2}{2\sum_{t=0}^{T-1}\eta_t} + \frac{4 d(B+\beta)}{p}\sqrt{\frac{2}{M}\log\frac{6}{3-5c}}\right.&\\
    \left.+4\sigma d\sqrt{\frac{2\log\frac{6}{c}}{{Mn}}} +4d\frac{(B+\beta)\sum_{t=0}^{T-1}\eta_t\tau(t)}{pM\sum_{t=0}^{T-1}\eta_t}\right]&
\end{align*}
\end{proof}
\begin{proof}[\bf Proof of Corollary \ref{corollary: main text learning rate}, \ref{corollary: gaussian learning rate}, and \ref{corollary: bounded learning rate}]
    For a constant learning rate $\eta_t=\frac{1}{\sqrt{dT}},$ plug it back in the corresponding inequalities, and we get the results.

    On the other hand, for a decaying learning rate $\eta_t=\frac{1}{\sqrt{d(t+1)}},$ we know that $\sum_{t=0}^{T-1}=\Theta\pth{\sqrt{T}}$ and $\sum_{t=0}^{T-1} \eta_t^2=\Theta\pth{\log T},$ plug them back in the corresponding inequalities, and we get the results in asymptotic.
\end{proof}
\section{Implementation Details and Additional Experiments}
\label{app: experiments general}
\subsection{Implementation Details}
\label{app:experiments}
\subsubsection{Datasets and preprocessing}
\label{sec: preprocessing}
\begin{itemize}
    \item {\bf MNIST \cite{lecun2009mnist}.} MNIST contains $60,000$ training images and $10,000$ testing images of $10$ classes.
    \item {\bf CIFAR-10 \cite{krizhevsky2009learning}.}  CIFAR-10 contains $50,000$ training images and $10,000$ testing images of $10$ classes.
\end{itemize}
\noindent{\bf Implementation.} We build our codes upon PyTorch \cite{paszke2019pytorch}. We run all the experiments with 8 GPUs of RTX A5000.
\subsubsection{Parameters}
\noindent{\bf Communication rounds}: $1500$ for both datasets in the section of client sampling, respectively, unless otherwise noted.

\noindent{\bf Dataset partition:} Clients' local datasets are evenly partitioned into balanced subsets. 
In Section \ref{sec: experiments alg 1} and \ref{sec: alg 2 learning}, we let each client own images from only two classes, which create a high non identical distribution. 
To characterize a different kind of non-IID distribution, we let clients' local data follow distribution with a concentration $\alpha=1$ in Section \ref{sec: alg 2 Byzantine}. Fig.\ \,\ref{fig: dirichlet} visualizes the impacts of different concentration parameter $\alpha$ on data distributions.
As $\alpha$ decreases, the local datasets become more and more non-IID across different clients.

\begin{figure}[!htb]
    \centering
    \includegraphics[width=\linewidth]{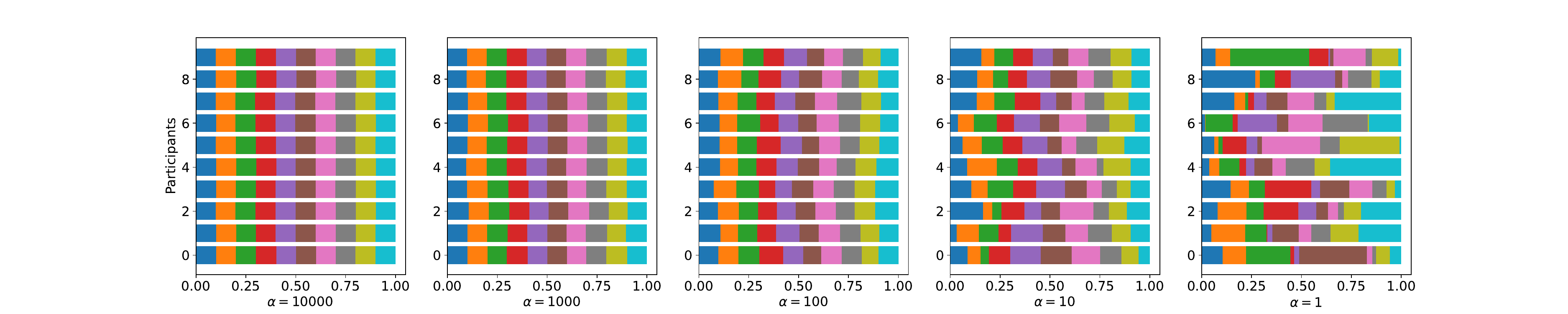}
    \caption{Dirichlet distribution with different concentration $\alpha$s}
    \label{fig: dirichlet}
\end{figure}

\noindent {\bf Mini-batch size.} We compare the peak performances of $\beta$-stochastic sign SGD under different mini-batch sizes through MLP. It is observed in Table \ref{table: batch size} that the Algorithm \ref{alg: alg 2} is not sensitive to mini-batch size $n$. This meets Remark \ref{rmk: convergence results}. 

\begin{table}
\centering
\begin{tabular}{ccc}
\hline
 Mini-batch Size & MNIST  & CIFAR-10 \\\hline
32      & 89.2\% & 46.03\%  \\
64      & 88.6\% & 46.68\%  \\
128     & 89.8\% & 46.78\%  \\
256     & 91.8\% & 46.54\% \\ \hline
\end{tabular}
\caption{Testing results on two datasets with different mini-batch sizes .}
\label{table: batch size}
\end{table}

\noindent{\bf Hyper parameters}: Mini-batch size is set as $n=32$ for both datasets. We consider a decaying learning rate of type $\eta_t=\frac{\eta_0}{\sqrt{t+1}}$ in Section \ref{sec: experiments alg 1} and \ref{sec: alg 2 learning}, while a constant learning rate $\eta_t=\eta_0$ in Section \ref{sec: alg 2 Byzantine} , and the initial choices are tuned through grid search. Specifically, $\eta_0\in\sth{0.0001,0.001, 0.006, 0.01, 0.03, 0.1}$, $B\in\sth{0.001,0.01,0.1,1}$ for $\beta$-Stochastic Sign SGD. 

\subsection{Test accuracy in Section \ref{sec: alg 2 learning}}
The observations in Section \ref{sec: alg 2 Byzantine} are consistent with Fig.\,\ref{fig: CIFAR learning acc}: 1) signSGD attains the worst performance; 2) an accuracy drop with the increase of privacy protection $\beta;$ 3) comparable accuracy between FedSGD and $\beta$-StoSign when privacy free.
%
%

\subsection{Byzantine adversary descriptions}
\label{app: byzantine}
In this section, we describe the Byzantine adversaries. We use the MLP network, Dirichlet distribution with concentration $\alpha=1$, and the same parameter settings as in Section \ref{sec: alg 2 learning} on MNIST dataset. The aggregation-rule-specific parameters are illustrated in the following part. All the experiment results are collected with $5$ repetitions.
\begin{itemize}
    \item {\bf Label flipping}: Suppose original label is $x$, the adversary will replace it with $9-x$;
    \item {\bf Inner Product Manipulation}: The adversaries send $-\frac{\gamma}{|\calN|}\sum_{i\in\calN}\nabla f(\bm{w}_i)$, instead of honest messages, to mislead the parameter server, where $\epsilon$ is the strength of the adversary. Let $\gamma=0.1$.
    \item {\bf A Little is Enough}: The adversaries estimate the benign clients' mean $\mu_\calN$ and standard deviation $\sigma_{\calN}$. Then, they will construct new messages as $\mu_{\calN}+z\sigma_{\calN}$ and upload to the parameter server, where $z$ is the strength of the adversary. We choose $z$ according to \cite{baruch2019little}:
    \begin{align*}
        z=\max_z\pth{\Phi(z)<\frac{M-s}{M}},
    \end{align*}
\end{itemize}
where $z=\lfloor\frac{M}{2}+1\rfloor-\abth{\calB(t)}$, and $\Phi$ is the cumulative distribution function of standard normal distribution. For us, $z\approx 0.5$.

\begin{figure}[!htb]
    \centering
    \begin{subfigure}[b]{\textwidth}
        \caption{$\beta=0$}
        \label{fig: CIFAR beta 0 acc}
        \includegraphics[width=\linewidth]{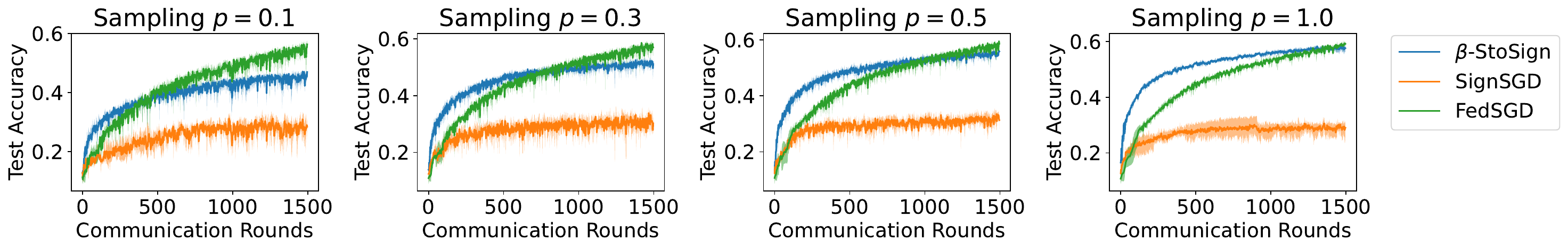}
    \end{subfigure}
    \begin{subfigure}[b]{\textwidth}
        \caption{$\beta=B$}
        \label{fig: CIFAR beta B acc}
        \includegraphics[width=\linewidth]{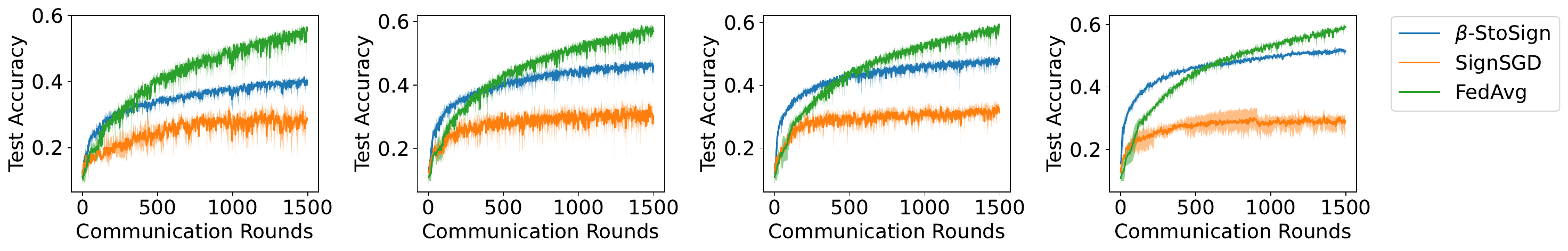}
    \end{subfigure}
    \begin{subfigure}[b]{\textwidth}
        \caption{$\beta=10B$}
        \label{fig: CIFAR beta 10B acc}
        \includegraphics[width=\linewidth]{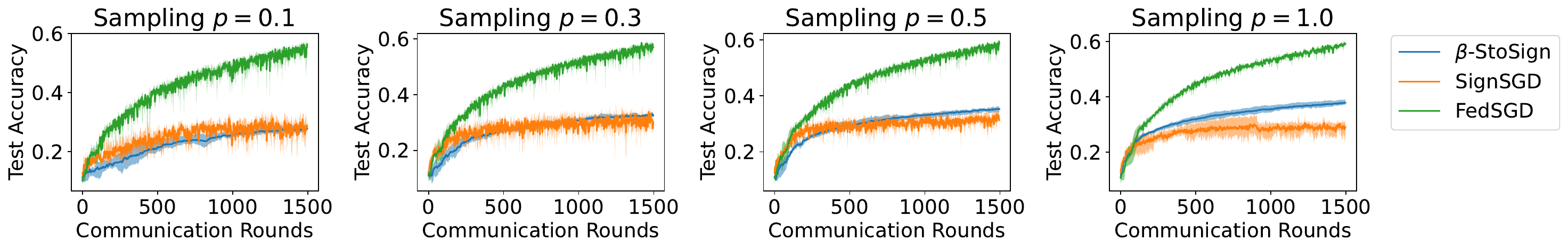}
    \end{subfigure}
    \caption{Test accuracy comparisons on CIFAR-10 data set under non-IID data}
    \label{fig: CIFAR learning acc}
\end{figure}
\begin{figure}[!htb]
    \centering
    \begin{subfigure}[b]{\textwidth}
    \caption{$\beta=0$}
        \includegraphics[width=\linewidth]{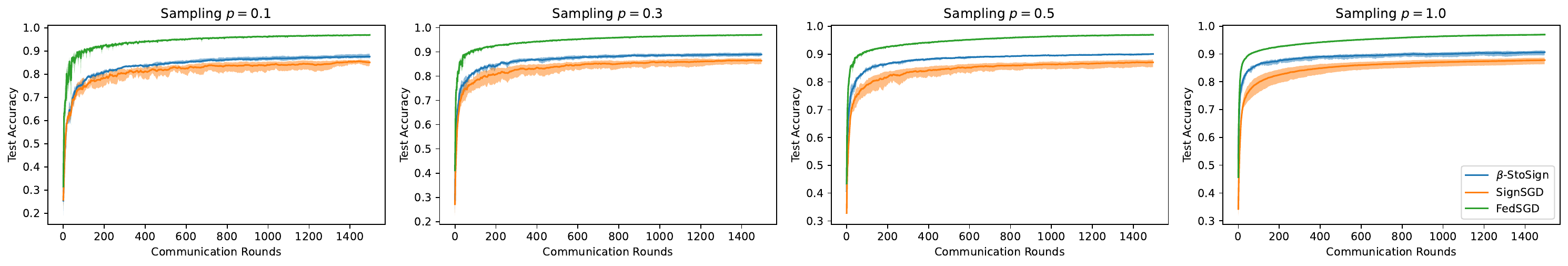}
    \end{subfigure}
    \begin{subfigure}[b]{\textwidth}
    \caption{$\beta=B$}
        \includegraphics[width=\linewidth]{fig/Learning_on_MNIST_beta_10_acc.pdf}
    \end{subfigure}
    \begin{subfigure}[b]{\textwidth}
    \caption{$\beta=10B$}
        \includegraphics[width=\linewidth]{fig/Learning_on_MNIST_beta_10_acc.pdf}
    \end{subfigure}
    \caption{Test accuracy comparisons on MNIST data set under non-IID data}
    \label{fig: MNIST learning acc}
\end{figure}



\end{document}